\documentclass[twoside]{article}

\usepackage{microtype}
\usepackage{graphicx}
\usepackage{amssymb, amssymb, amsthm}
\usepackage{amsmath}
\usepackage{color}
\usepackage{booktabs}
\usepackage{enumerate}
\usepackage{lscape}
\usepackage{longtable}
\usepackage{mathrsfs}
\usepackage{rotating}
\usepackage{subfigure}
\usepackage{tcolorbox}
\usepackage{thmtools}
\usepackage{thm-restate}
\usepackage{hyperref}
\usepackage{algorithm}
\usepackage{algorithmic}
\usepackage{enumitem}
\usepackage{selectp}

\usepackage[accepted]{aistats2022}
\usepackage[round]{natbib}
\setcitestyle{authoryear,open={(},close={)}}

\def\Um{{\bf U}}

\newcommand{\D}{{\mathcal D}}
\newcommand{\E}{{\mathbb E}}

\newcommand{\N}{{\mathcal N}}
\newcommand{\R}{{\mathbb{R}}}
\renewcommand{\S}{{\mathcal S}}

\newtheorem{theorem}{Theorem}
\newtheorem{lemma}[theorem]{Lemma}

\newtheorem{assumption}{Assumption}
\newtheorem{definition}{Definition}
\newtheorem{remark}{Remark}

\def\real{\mathbb{R}}
\DeclareMathOperator{\es}{ES}
\DeclareMathOperator{\trial}{trial}

\DeclareMathOperator{\Hc}{H}
\DeclareMathOperator{\cl}{CL}
\DeclareMathOperator{\argmax}{argmax}

\newtcolorbox{mybox}[2][]
{
  colframe = #2!25,
  colback  = #2!25!white!25,
  left=1mm,
  top=1mm,
  #1
}

\begin{document}

% If your paper is accepted and the title of your paper is very long,
% the style will print as headings an error message. Use the following
% command to supply a shorter title of your paper so that it can be
% used as headings.
%
%\runningtitle{I use this title instead because the last one was very long}

% If your paper is accepted and the number of authors is large, the
% style will print as headings an error message. Use the following
% command to supply a shorter version of the authors names so that
% they can be used as headings (for example, use only the surnames)
%
%\runningauthor{Surname 1, Surname 2, Surname 3, ...., Surname n}

\twocolumn[

\aistatstitle{A Globally Convergent Evolutionary Strategy for Stochastic Constrained Optimization \\ with Applications to Reinforcement Learning}

\aistatsauthor{ Youssef Diouane$^*$ \And Aurelien Lucchi$^*$ \And Vihang Patil$^*$}

\aistatsaddress{Department of Mathematics \\ and Industrial Engineering \\ Polytechnique Montr\'eal \\ \texttt{\small youssef.diouane@polymtl.ca} \And Department of Mathematics \\ and Computer Science \\ University of Basel  \\ \texttt{\small aurelien.lucchi@unibas.ch} \And Institute for Machine Learning \\ Johannes Kepler University Linz   \\ \texttt{\small patil@ml.jku.at}} ]

\begin{abstract}
Evolutionary strategies have recently been shown to achieve competing levels of performance for complex optimization problems in reinforcement learning. In such problems, one often needs to optimize an objective function subject to a set of constraints, including for instance constraints on the entropy of a policy or to restrict the possible set of actions or states accessible to an agent. Convergence guarantees for evolutionary strategies to optimize \emph{stochastic} constrained problems are however lacking in the literature. In this work, we address this problem by designing a novel optimization algorithm with a sufficient decrease mechanism that ensures convergence and that is based only on estimates of the functions. We demonstrate the applicability of this algorithm on two types of experiments: i) a control task for maximizing rewards and ii) maximizing rewards subject to a non-relaxable set of constraints.
\end{abstract}

% !TEX root = main.tex

\section{INTRODUCTION}
\vspace{-1mm}

% Gradient-based vs ES
Gradient-based optimization methods are pervasive in many areas of machine learning. This includes deep reinforcement learning (RL) which is notoriously known to be a challenging task due to the size of the search space as well as the problem of delayed rewards.  The optimization landscape is also known to have lots of irregularities, where gradients can be extremely small in magnitude~\citep{agarwal2019optimality}, which can severely hinder the progress of gradient-based methods. In order to overcome such difficulties, one needs to be able to efficiently explore the search space of parameters, which partially explains the recent success of a class of global optimization methods known as evolutionary strategies (ES) in reinforcement learning~\citep{maheswaranathan2018guided}. These methods belong to the class of randomized search that directly search the space of parameters without having to explicitly compute any derivative. Starting from an initial parameter vector $x_0$, the algorithm samples a set of offsprings near $x_0$. Based on the objective function values, the best offsprings are selected to update the parameter $x_0$. Multiple variants of ES methods have been proposed in the literature, including for instance the covariance matrix adaptation evolution strategy (CMA-ES)~\citep{hansen2001completely} as well as natural evolutionary strategies~\citep{wierstra2008natural}.

% Convergence guarantees

Given the recent attention given to evolutionary strategies in reinforcement learning, the question of global convergence of these methods seems of both theoretical and practical interest. By global convergence, we mean convergence to a first-order stationary point independently of the starting point.
One approach for proving global convergence is to modify the traditional ES algorithm by accepting new iterates based on a forcing function that requires a sufficient amount of decrease at each step of the optimization process \citep{YDiouane_SGratton_LNVicente_2015_a}. A similar paradigm can be adapted to constrained problems \citep{YDiouane_SGratton_LNVicente_2015_b, YDiouane_2021}. More recently, for a simple instance of ES where recombination is not considered, \cite{Glasmachers:2020} showed a form of the global convergence can be achieved without imposing a sufficient decrease conditions on the population. The guarantees provided by these methods are however not applicable to typical practical problems in machine learning where the objective function (and potentially the constraints) can not be evaluated exactly, either for computational reasons, or because of the existence of inherent noise. We address this problem taking as a core motivation the problem of reinforcement learning where an agent learns to act by a process of trial and error which over time allows it to improve its performance at a given task. While early work in reinforcement learning allowed the agent to freely explore actions, more recent work, e.g.~\cite{achiam2017constrained}, has advocated for the use of constrained policies. As pointed out in~\cite{achiam2017constrained}, this is critical in certain environments such as robot automation for industrial or medical applications. Another typical example of constraints that are commonly found in RL are for maximizing the entropy of a policy \citep{haarnoja2018soft}. Without such constraints, one might converge to a local solution that is far away from any global optimum. More examples of constrained problems for safety purposes can be found in~\citep{raybenchmarking}. Motivated by RL applications, the proposed approach in this paper extends the works \citep{YDiouane_SGratton_LNVicente_2015_a,YDiouane_SGratton_LNVicente_2015_b} to the setting where only stochastic estimates of the objective and the constraints  are available.

Our main goal is to design a variant of an evolutionary strategy with provable convergence guarantees in a constrained and stochastic setting. Broadly, the problem we consider can be cast as the following general stochastic optimization problem:
\begin{equation}
\label{c:Prob1}
\begin{split}
\min_x \quad & f(x) ~~~~~\mbox{s.t.} \quad   x \in \Omega,
\end{split}
\end{equation}
where the objective function $f$ is assumed  to be continuously differentiable.
The feasible region $\Omega \subset \real^n $ will be assumed, in the context of this paper, to be of the form:
\begin{equation}
\Omega = \left\{ x \in \real^n |  \forall i \in \{1,\dots, r\} , c_i(x) \leq 0 \right\},
\label{eq:Omega}
\end{equation}
where each $c_i: \Omega \to \R$ is a given constraint function. We note that both linear and non-linear constraints can adequately been incorporated in $\Omega$.

In this paper, our main contributions are as follows:
\vspace{-2mm}
\begin{itemize}[leftmargin=2em]
\itemsep0em 
\item The design of a variant of an evolutionary strategy with provable convergence guarantees. While prior work in reinforcement learning such as~\cite{salimans2017evolution, asebo} has demonstrated the good empirical performance of evolutionary strategies, it does not provide convergence guarantees.
\item The theoretical guarantees we derive apply to unconstrained and constrained stochastic problems. While convergence guarantees for ES exist for optimization problems with **exact** function values (for the objective and the constraints), we are not aware of any prior work that handles stochastic problems where only estimates of the objective and the constraints are available.
\item We test the empirical performance of our approach on a variety of standard RL problems and observe higher returns compared to common baselines. Importantly, our algorithm guarantees the feasibility of the constraints, which might be extremely important in some environments.
\end{itemize}

% !TEX root = main.tex

\vspace{-2mm}
\section{RELATED WORK}

\vspace{-2mm}
\paragraph{Evolutionary strategies in RL}

Classical techniques to solve RL problems include methods that use trajectory information such as policy gradients or Q-learning \citep{sutton2018}. One alternative to these techniques is to use black-box optimization methods such as random search techniques. There has recently been a renewed interest in such methods, especially in the context of deep RL problems where they have been shown to be scalable to large problems ~\citep{mania2018simple, salimans2017evolution, asebo}. For instance, ~\cite{salimans2017evolution} showed that ES can be scaled up using distributed systems while~\cite{maheswaranathan2018guided} suggested to use surrogate gradients to guide the random search in high-dimensional spaces. Another advantage ES methods have is that they are not affected by delayed rewards \citep{arjonamedina2019rudder, patil_2020}. Because evolutionary methods learn from complete episodes, they tend to be less sample efficient than classical deep RL methods. This problem has been addressed in prior work, including e.g. \cite{pourchot2018} who suggested an approach named CEM-RL that combines an off-policy deep RL algorithm with a type of evolutionary search named Cross Entropy Method (CEM). The combination of these  approaches makes CEM-RL able to trade-off between sample efficiency and scalability. \cite{khadkaandtumer2018} proposed another sample efficient hybrid algorithm where they utilize gradient information by adding an agent trained using off-policy RL into the evolving population at some fixed interval.  \cite{TRES2019} improve the sample efficiency of ES by using a trust region approach that optimizes a surrogate loss, enabling to reuse data sample for multiple epochs of updates. \cite{conti2017improving} improve the exploration qualities of ES for RL problems by utilizing a population of novelty seeking agents. Further, ES has also been used to evolve policies in model-based RL \citep{ha2018world}. 

\vspace{-2mm}
\paragraph{Constrained optimization in RL} 
~\cite{achiam2017constrained} designed an algorithm to optimize the return while satisfying a given set of deterministic constraints. Their approach relies on a trust-region method which has been shown in~\citep{
salimans2017evolution} to be practically outperformed by evolutionary strategies in various environments. A similar approach was proposed by~\cite{tessler2018reward} but for a larger set of deterministic constraints. Further, \cite{chow2019lyapunovbased} use Lyapunov constraints to obtain feasible solutions on which the policy or the action is projected to guarantee the satisfaction of constraints.
Another application of constrained optimization is to enforce safety rules in an RL environment. For instance, an agent exploring an environment might not want to visit certain states that are deemed unsafe. This problem has been formalized in~\cite{Altman1999ConstrainedMD} which will be discussed in more details in Sec.~\ref{sec:experiments}. In this paper, to handle constraints, we extend the unrelaxable constraints methodology, as in \citep{CAudet_JEDennis_2006}, to include uncertainties in  the estimates of the objective function and the constraints. In particular, in our context, the constraints will be handled using an adjusted extreme barrier function (see Section \ref{sec:method}). 

\vspace{-2mm}
\paragraph{Maximum Entropy RL} 
Entropy maximization in RL has been claimed to connect local regions in the optimization landscape, thereby making it smoother \citep{ahmed2018understanding}, which enables faster learning and also better exploration. Recent prior work include Soft Actor-Critic~\citep{haarnoja2018soft}, Soft Q-learning~\citep{haarnoja2017reinforcement}.

\vspace{2mm}
To the best of our knowledge, none of the works discussed above provided convergence guarantees for an ES algorithm in a stochastic constrained setting. As we will see shortly, this will require a new Lyapunov function that is different from the one used for deterministic methods, e.g.~\cite{YDiouane_SGratton_LNVicente_2015_a}.
% !TEX root = main.tex

\vspace{-1mm}
\section{METHOD}
\label{sec:method}
\vspace{-2mm}

\subsection{The proposed framework}
\vspace{-1mm}

\paragraph{Provably convergent ES}
Evolution strategies iteratively sample candidate solutions from a distribution $\D_k$ (scaled by a factor $\sigma_k^{ES} > 0$) and select the best subset of candidates to create an update direction $d_k$. The next iterate is then given by $x_{k+1}^{\trial} = x_k + \sigma_k d_k$ where $\sigma_k$ is a step-size parameter. A general technique~\citep{YDiouane_SGratton_LNVicente_2015_a} to ensure this approach globally converges is by imposing a sufficient decrease condition on the objective function value, which forces the step size~$\sigma_k$ to converge to zero.
Constrained problems are discussed in~\cite{YDiouane_SGratton_LNVicente_2015_b}, which starts with a feasible iterate $x_0$ and prevents stepping outside the feasible region by means of a barrier approach. In this context, the sufficient decrease condition is applied not to $f$ but to the extreme barrier function $f_{\Omega}$ associated to $f$ with respect to the constraints set $\Omega$~\citep{CAudet_JEDennis_2006} (also known as death penalty function), which is defined by:
\begin{equation}
\label{extreme}
f_{\Omega}(x) \; = \; \left\{\begin{array}{ll}f(x) & \textrm{if } c(x)\le 0 \quad  \\ +\infty & \textrm{otherwise}
\end{array}\right.
\end{equation}
where $c(x)$ is a constraint function as defined in Eq.~\ref{eq:Omega}.
\vspace{-3mm}
\paragraph{Inexact function values and constraints}
In this work, we consider the case where the function values 
cannot be accessed exactly and only some estimates of the objective function and the constraints are available. The definition of the barrier function evaluated at a point $x_k$ is adjusted as follows: 
\begin{equation}
\label{extreme_2}
\tilde f_k \; = \; \left\{\begin{array}{ll}f_k & \textrm{if } c_k -\varepsilon_c \sigma_k \le 0 \\ +\infty & \textrm{otherwise,}
\end{array}\right.
\end{equation}
where $f_k$ and $c_k$ are the estimation of $f$ and $c$ at the point $x_k$, and $\varepsilon_c>0$ a fixed tolerance on the constraints. The obtained method is thus given by Algorithm~\ref{alg:GESgc}.
%%%%%%%%%%%%%%%%%%%%%%%%%%%%%%
\begin{algorithm*}[!ht]
%\SetAlgoNlRelativeSize{0}
\caption{\bf \bf: A class of ES using estimates}
\label{alg:GESgc}
\begin{algorithmic}[1]
\STATE Choose positive integers $\lambda$ and $\lambda'$ such that $\lambda \geq \lambda'$.
 Choose initial step lengths $\sigma_0,\sigma_0^{\es} > 0$ and the constants $\gamma,~d_{\max}$ such that $\gamma \ge 1$ and $d_{\max}>0$. Select two positive constants $\epsilon_c>0$ and $\kappa>0$.
 Select an initial $x_0 \in \real^n$ such that $c_0 \le \epsilon_c \sigma_0$ and evaluate $ f_0 < \infty$ the estimation of $f$ at $x_0$.
\FOR{$k= 0,  1, \ldots$}

\STATE \textbf{Step 1:} compute new sample points $Y_{k+1} = \{ y_{k+1}^1,\ldots,y_{k+1}^{\lambda} \}$
such that, for all $ i=1,\ldots,\lambda$, one has $y_{k+1}^i  =  x_k + \sigma_k^{\es} d_k^i$
where the directions $d_k^i$'s  are drawn from a distribution $\mathcal{D}_k$.
 
\STATE \textbf{Step 2:} compute $\{ f_{k+1}^i \}_{i=1,\ldots,\lambda}$ the estimates of $f$ at the set point  $Y_{k+1}$ and re-order the offspring points of $Y_{k+1}$ into $\tilde Y_{k+1} = \{ \tilde{y}_{k+1}^1,\ldots,\tilde{y}_{k+1}^{\lambda} \}$
by increasing order: $\tilde f_{k+1}^1 \leq \cdots \leq \tilde f_{k+1}^{\lambda}$ where $\tilde f_{k+1}^i$ is the estimation of $f$ at the sample point $\tilde{y}_{k+1}^i$.

\STATE \textbf{Step 3:} select the new parents as the best $\lambda'$ offspring sample points
$\{ \tilde{y}_{k+1}^1,\ldots,\tilde{y}_{k+1}^{\lambda'} \}$ and let $\{ \tilde{d}_{k}^1,\ldots,\tilde{d}_{k}^{\lambda'} \}$ be the associated directions. 
Set $d_k=\Psi_k(\tilde{d}_{k}^1,\ldots,\tilde{d}_{k}^{\lambda'})$ where $\Psi_k$ is a linear mapping related to the chosen ES strategy
 (such that  $\|d_k \| \leq d_{\max}$). Let
\begin{equation}
x_{k+1}^{\trial} \; = \; x_k + \sigma_k d_k
\label{eq:trial_point}
\end{equation}
and $\tilde f^{\trial}_{k+1}$ be the estimate of the barrier function at $x_{k+1}^{\trial}$ using \eqref{extreme_2}.\;

\STATE \textbf{Step 4:}
\IF{$\tilde f^{\trial}_{k+1} \leq f_k - \frac{\kappa}{2}\sigma^2_k$}
\STATE Consider the iteration successful, set $x_{k+1} = x_{k+1}^{\trial}$, $ f_{k+1} = \tilde f^{\trial}_{k+1}$, and
$\sigma_{k+1} = \gamma \sigma_k$.\;
\ELSE
\STATE Consider the iteration unsuccessful, set $x_{k+1} = x_k$, $ f_{k+1} =  f_{k}$ and
$\sigma_{k+1} = \gamma^{-1} \sigma_k$.\;
\ENDIF

\STATE \textbf{Step 5:} update the ES parameters (i.e., $ \sigma_{k+1}^{\es}$ and $\mathcal{D}_{k+1}$).\; 
\ENDFOR
\end{algorithmic}
\end{algorithm*}
%%%%%%%%%%%%%%%%%%%%%%%%%%%%%%

\vspace{-2mm}
\paragraph{Algorithm}
The first two steps sample a set of $\lambda > 0$ candidate directions and rank them according to their corresponding function values. In step 3, the algorithm combines the best subset of these directions (of size $\lambda' > 0$) using a linear mapping $\Psi_k$ whose choice depends on the chosen ES strategy. For instance, using a CMA-ES strategy as proposed by \cite{hansen2001completely}, the mapping $\Psi_k$ is a simple averaging function, i.e. $\Psi_k(\tilde{d}_{k}^1,\ldots,\tilde{d}_{k}^{\lambda'})=\sum_{i=1}^{\lambda'} w^i_k \tilde{d}_{k}^i$ where the weights $\{w^i_k\}_i$ belong to a simplex set. Another example of mapping function $\Psi_k$ is the Guided ES \citep{maheswaranathan2018guided}, where one typically has $\lambda=2\lambda'$ and $\Psi_k$ is given by $\Psi_k(\tilde{d}_{k}^1,\ldots,\tilde{d}_{k}^{\lambda'})=\frac{1}{\lambda} \sum_{i=1}^{\lambda'} \left( (f_{k+1}^i -  f_{k+1}^{i+\lambda'})/\sigma_k^{\es} \right) \tilde{d}_{k}^i$. For further details, we refer the reader to the appendix.
The direction computed by $\Psi_k$ is denoted by $d_k$. The algorithm steps in the direction $d_k$ using a step size $\sigma_k$, which is then adjusted in step 4 depending on whether the iteration decreases the function or not. We note that, for generality reasons, the updates of the ES parameters ($\sigma^{\es}_{k+1}$ and $\mathcal{D}_{k+1}$) in step 5 are purposely left unspecified. In fact, our convergence analysis is independent of the choice of the sequences $\{\sigma^{\es}_{k}\}_k$ and $\{\mathcal{D}_{k}\}_k$. For the experimental results reported in Section~\ref{sec:experiments}, we use the same update rule for $\sigma^{\es}_{k}$ as $\sigma_{k}$. 

\begin{remark}[Extreme barrier vs projection]
In some applications, the feasible set is formed with linear constraints or simple bounds. In such cases where a projection to the feasible domain is computationally affordable,  the use of the barrier function given by (\ref{extreme_2}) in Algorithm~\ref{alg:GESgc} can be replaced by a projection. As long as the sufficient decrease condition is enforced, our convergence theory applies. Using exact estimates, \cite{YDiouane_SGratton_LNVicente_2015_b} showed that the analysis for both an extreme barrier approach and a projection approach to handle constraints are equivalent.
\end{remark}
\vspace{-1mm}
\begin{remark}[Analysis unconstrained case (new result)]
Although Algorithm~\ref{alg:GESgc} is presented for constrained problems, the adaptation to the unconstrained case is straightforward. Indeed, it suffices to replace the barrier function estimates Eq.~\eqref{extreme_2}, computed at the offspring points and at the trial point $x_{k+1}^{\trial}$  by estimates of the objective function at the same points. The convergence analysis of the unconstrained framework can be deduced from the analysis we derived in the constrained case. We emphasize that, to the best of our knowledge, the analysis for the stochastic unconstrained case is also a new result in the literature.
\end{remark}

\vspace{-4mm}
\subsection{Accuracy of the estimates}
\vspace{-3mm}
In order to obtain convergence guarantees for Algorithm~\ref{alg:GESgc}, we require the estimates of $f$ to be sufficiently accurate with a suitable probability.
For practical reasons, we are interested in the case where the directions in Algorithm~\ref{alg:GESgc} are not defined deterministically but generated by a random process defined in a probability space $(\mathcal{E}, \mathcal{F}, P)$. Note that the randomness of the direction implies the randomness of the iterate  
$X_k$, the direction $D_k$, the  parameters $\Sigma_k$ and $S_k= \Sigma_k D_k$. Given a sample $\omega \in \mathcal{E}$, we denote by $d_k= D_k(\omega)$, $x_k= X_k(\omega)$,  $\sigma_k= \Sigma_k(\omega)$, and $s_k= S_k(\omega)$ their respective realizations.
Moreover, the objective function $f$ and the constraints are supposed to be accessed only through stochastic estimators. Therefore,  we define the realizations of the random variables $F_k^0$ (i.e., the estimate of the objective function $f$ at the iterate $X_k$) and $F_k^1$ (i.e., the estimate of the objective function $f$ at the iterate $X_k+S_k$) by $f_k^0$ and $f_k^1$ respectively. Similarly, we denote the realizations of the constraints $C_k^0$ (i.e., the estimate of the constraints $c$ at the iterate $X_k$) and $C_k^1$ (i.e., the estimate of the constraints $c$ at the iterate $X_k+S_k$) by $c_k^0$ and $c_k^1$.
As mentioned earlier, we will require the random estimates to have a certain degree of accuracy during the application of the proposed framework. The accuracy of the objective functions estimates is formalized below.

\begin{definition} \label{defi:probaccfun}
Given constants  $\varepsilon>0$, and $p \in (0,1]$, the sequence
of the random quantities $F_k^0$ and $F_k^1$ is  called 
$p$-probabilistically $\varepsilon_f$-accurate, for corresponding 
sequences $\{ X_k\}$, $\{ \Sigma_k\}$, if the event
\begin{align*}
	T^f_k \; \stackrel{\Delta}{=} \;  \big\{ \left|F_k^0 - f(X_k) \right| \; \le \;
	\varepsilon_f \Sigma_k^{2} \\
	~\mathrm{and}~ \left|F_k^1 - f(X_k+S_k) \right| \; \le \;
	\varepsilon_f \Sigma_k^{2} \big \}\end{align*}
satisfies the condition
$\mathbb{P}\left(~T^f_k~~| \mathcal{F}_{k-1} \right) \geq  p
$, where $\mathcal{F}_{k-1}$ is the $\sigma$-algebra generated by 
 $F^0_0,F^1_0 \ldots, F_{k-1}^0,F_{k-1}^1$  and $C^0_0,C^1_0 \ldots, F_{j-1}^0,C_{k-1}^1$.
\end{definition}
In the context of this paper, the estimates of the constraints will be assumed to be almost-surely accurate as $\Sigma_k \to 0$ in the following sense:
\begin{definition} \label{defi:probaccconstr}
Given a constant  $\varepsilon_{c}>0$, the sequence
of the random quantities $C_k^0$ and $C_k^1$ is  called 
almost-surely $\varepsilon_c$-accurate, for corresponding 
sequences $\{ X_k\}$, $\{ \Sigma_k\}$, if the event
\vspace{-2mm}
\begin{align*}
	T^c_k \; \stackrel{\Delta}{=} \;  \big\{ \left\|C_k^0 - c(X_k) \right\|_{\infty} \; \le \;
	\varepsilon_c \Sigma_k \\
	\quad \mathrm{and} 
	\quad \left\|C_k^1 - c(X_k+S_k) \right\|_{\infty} \; \le \;
	\varepsilon_c \Sigma_k \big \}
\end{align*}
satisfies the condition
$ \mathbb{P}\left(~T^c_k~~| \mathcal{F}_{k-1} \right) =  1$, 
where $\mathcal{F}_{k-1}$ is the $\sigma$-algebra generated by 
 $F^0_0,F^1_0 \ldots, F_{k-1}^0,F_{k-1}^1$ and $C^0_0,C^1_0 \ldots, C_{k-1}^0,C_{k-1}^1$.
\end{definition}
In Definition~\ref{defi:probaccfun}, the accuracy of the function estimation gap is of order $\Sigma_k^2$, which is a common assumption in the literature, see e.g.~\cite{BlanchetCartisMenickellyScheinberg19}. For the constraints, our analysis will require only to have estimates that converge to the exact value as $\Sigma_k \to 0$. For simplicity reasons, we make the choice of using only $\Sigma_k$ in Eq.\eqref{extreme_2} and Definition \ref{defi:probaccconstr} to measure the accuracy level of the constraints. That can be generalized to take the  form $\|C^0_k - c(X_k) \|_\infty \to 0$ and $\left\|C_k^1 - c(X_k+S_k) \right\|_{\infty} \to 0$ as $\Sigma_k \to 0$.

\vspace{-2mm}
\subsection{Global convergence}
\vspace{-1mm}

We derive a convergence analysis of Algorithm~\ref{alg:GESgc} under the following  assumptions.
\begin{assumption} \label{ass:f}
$f$ is continuously differentiable on an open set containing the level set $\mathcal{L}(x_0) = \{x \in \R^n \ | \ f(x) \leq f(x_0)\}$, with Lipschitz continuous gradient, of Lipschitz constant $L$. 
\end{assumption}
\begin{assumption} \label{ass:bdn:f}
$f$ is bounded from below by $ f_{\mathrm{low}}$. 
\end{assumption}
\begin{assumption} \label{ass:ProbAccEstimates:obj}

The sequence of random objective function estimates $ \{F_k^0, F_k^1\}_k$ generated by Algorithm~\ref{alg:GESgc} satisfies the two following conditions:

(1) The sequence $ \{F_k^0, F_k^1\}_k$ is $p$-probabilistically $\varepsilon_f$-accurate for some $p \in (\frac{1}{2},1]$, $\varepsilon_f \in  (0,\frac{\kappa}{4})$ where $\kappa$ is a constant used in Algorithm~\ref{alg:GESgc}.

(2) There exists  $\varepsilon_v>0$ such that the sequence of estimates $ \{F_k^0, F_k^1\}_k$ 
satisfies the following $\varepsilon_v$-variance condition for all $k\ge 0$,
%\vspace{-.5cm}
\begin{align*}
	\mathbb{E}\left(\left|F_k^0 - f(X_k) \right|^2  | \mathcal{F}_{k-1} \right) \;\le\; \varepsilon_v^2 \Sigma_k^4 \\
	\quad \mathrm{and}
	\quad \mathbb{E}\left(\left|F_k^1 - f(X_k+S_k) \right|^2  | \mathcal{F}_{k-1} \right) \;\le\; \varepsilon_v^2 \Sigma_k^4.
\end{align*}
%\end{enumerate}
\end{assumption}

\begin{assumption}
\label{ass:AccEstimates:cons}
For all $k$, the sequence of random constraints estimates $ \{C_k^0, C_k^1\}_k$ generated by Algorithm~\ref{alg:GESgc} is $\varepsilon_c$-accurate almost surely, for a given constant $\varepsilon_c > 0$.
\end{assumption}

\vspace{-2mm}
\paragraph{Existence of a converging subsequence}

For the sake of our analysis, we introduce the following (random) Lyapunov function 
\begin{equation}
\Phi_k \ 	= \ 	\nu f(X_k) + (1 - \nu) \Sigma_k^2,
\label{eq:Lyapunov}
\end{equation}
where $\nu \in (0, 1)$.
Consider a realization of Algorithm~\ref{alg:GESgc}, and let $\phi_k$ be the corresponding realization of $\Phi_k$.
The next theorem shows that, under Assumption \ref{ass:ProbAccEstimates:obj}, the imposed decrease condition, in Algorithm~\ref{alg:GESgc} leads to an expected decrease on the Lyaponov function $\Phi_k$. 

\begin{mybox}{gray}
\begin{restatable}{theorem}{DecreaseOnPhi}
\label{thm:DecreaseOnPhi}
Let Assumption \ref{ass:f}  hold. Suppose that Assumption \ref{ass:ProbAccEstimates:obj} is also satisfied with
probability $p$ such that $
\frac{p}{(1-p)^{1/2}} \ge \frac{4\nu \varepsilon_v}{(1-\nu)(1-\gamma^{-2})}$
and $\nu \in (0, 1)$ satisfies
$\frac{\nu}{1 - \nu} \  \ge  \frac{4(\gamma^2 - 1)}{\kappa}$.

 Then, there exists an $\alpha>0$ such that, for all $k$,
\begin{equation} \label{eq:MainGoal}
	\mathbb{E}\left(\Phi_{k+1} - \Phi_k | \mathcal{F}_{k-1}\right)\leq - \alpha \Sigma_k^2.
\end{equation}
\end{restatable}
\end{mybox}

Hence, the true value of the objective function $f$ may not decrease at \emph{each} individual iteration but Theorem \ref{thm:DecreaseOnPhi} ensures that the Lyapunov function decreases over iterations in expectation as far as the accuracy probability of the estimates of $f$ are high enough. Using such result, one can guarantee
that the sequence of step sizes $\{\sigma_k\}$ will converge to zero almost surely. In particular, this will ensure the existence of a subsequence~$\mathcal{K}$ of iterates driving the step size to zero almost surely.
Then, assuming boundedness of the iterates, it will be possible to prove the existence of a convergent subsequence.

\begin{mybox}{gray}
\begin{restatable}{corollary}{StepSizeToZeroAlmostSurely}
\label{corol:StepSizeToZeroAlmostSurely}
Let Assumption \ref{ass:bdn:f} hold. Suppose that the working assumptions of Theorem \ref{thm:DecreaseOnPhi} are also satisfied. 
Then,  the sequence $\{\Sigma_k\}_k$ goes to $0$ almost surely. Moreover, if the sequence of iterates $\{X_k\}$ is bounded, then there exists a  subsequence $\mathcal{K}$ and $X_*$ such that $\{\Sigma_k\}_{k \in \mathcal{K}}$ goes to zero and  $\{X_k\}_{k\in  \mathcal{K}}$ converges almost surely to $X_*$.
\end{restatable}
\end{mybox}

Now that we have established the existence of a converging subsequence, a natural question is to study the properties of its limit point. In our case, we are interested in showing that this limit point satisfies the desired optimality condition for constrained problems, which we briefly review next.

\vspace{-2mm}
\paragraph{Optimality conditions for constrained problems} In optimization, first-order optimality conditions for constrained problems can be defined by using the concept of tangent cones. A known result is that the gradient at optimality belongs to the tangent cone (see , e.g., Thm 5.18~\cite{VariationalAnalysis_book}).
In order to prove that Algorithm~\ref{alg:GESgc} satisfies the desired first-order optimality condition, we will require that for iterates $x \in \Omega$ arbitrarily close to~$x_*$, the updated point $x+td$ (for $t > 0$ and a fixed direction $d$) also belongs to the constraint set $\Omega$. This can simply be guaranteed by ensuring that the set of the directions $d$ is hypertangent to $\Omega$ at $x_*$ \citep{CAudet_JEDennis_2006}. For readers who are not familiar with constrained optimization, we give an overview of the required concepts and definitions in Section~\ref{sec:optimality_condition} in the appendix.

\vspace{-2mm}
\paragraph{Main Convergence Theorem}

We now state the main global convergence result for Algorithm~\ref{alg:GESgc}. A formal variant of this theorem is available in appendix.

\begin{mybox}{gray}
\begin{theorem}[Informal]
\label{th:1}
Assume Assumption \ref{ass:AccEstimates:cons} and the working assumptions of Corollary \ref{corol:StepSizeToZeroAlmostSurely} hold.
Then, almost surely, the limit point of the subsequence of iterates $\{ X_k(\omega) \}_K$ satisfies the desired optimality condition for constrained problem and $\lim_{k \in K} \Sigma_k(\omega) = 0$.
\end{theorem}
\end{mybox}

Theorem~\ref{th:12} (appendix) gives the formal statement of Theorem~\ref{th:1}. Theorem~\ref{th:12} states that, almost surely, the gradient at a limit point (of the algorithm iterate points) satisfies the desired optimality condition, i.e. it makes an acute angle with the tangent cone of the constraints. We therefore have shown convergence of Algorithm~\ref{alg:GESgc} to a point that is guaranteed to satisfy the desired constraints under a set of assumptions that, as discussed below, can be achieved in practice.

\vspace{-2mm}
\paragraph{Remark about satisfiability of the assumptions:}
We note that the differentiability and boundedness assumptions (Assumptions~\ref{ass:f} and~\ref{ass:bdn:f}) are common in machine learning. While the satisfaction of the accuracy required in Assumptions \ref{ass:ProbAccEstimates:obj} and \ref{ass:AccEstimates:cons} might appear less trivial, one can in fact easily derive practical bounds in the context of finite sum minimization problems, e.g. Lemma 4.2~\cite{EBergou_YDiouane_VKungurtsev_CWRoyer_2018b}.

For instance, one can perform multiple function evaluations and average them out. We therefore get an estimate $f_x = \frac{1}{N}\sum_{i=1}^N f_i(x)$, where the set $\{f_1,...f_N\}$ correspond to independent samples. Assuming bounded variance, i.e. $\text{var}(f(x)) \le v$, known concentration results guarantee that we can obtain $p$-probabilistically $\epsilon_f$-accurate estimates for $ N \geq \frac{16v}{\epsilon_f^2 \sigma_k^4}\log(\frac{2}{1 - p} ) $ number of evaluations. To also satisfy the variance assumption, we additionally require $N \geq \frac{v}{\epsilon_f \sigma_k^4}$. We also note that one could still violate Assumption \ref{ass:ProbAccEstimates:obj} to some degree and obtain convergence to a neighborhood around the optimum.

In the presence of constraints, Assumption~\ref{ass:AccEstimates:cons} allows the use of  inaccurate estimations of the constraints, in particular during the early stage of the optimization process (when $\sigma_k$ may be large). In practice, our framework can also be seen suitable to  handle hidden constraints (e.g., code failure) as far as one assumed that the set of such constraints is of $P$-measure zero. A second practical scenario is related to equality constraints where the feasible domain can be hard to reach using an evolution strategy. Using Assumption~\ref{ass:AccEstimates:cons} allows us to relax the feasible domain. For instance, an equality constraint of the form $c(x)=0$, by using Assumption~\ref{ass:AccEstimates:cons}, will be handled as $|c(x)| \le \epsilon_c \sigma_k $. This relaxation can gradually help the ES to explore  the feasible domain and improve the objective function. The empirical validity of the assumptions on our test cases is given in the appendix, see Figure~\ref{fig:variance_1} and \ref{fig:variance_2}.

\vspace{-2mm}
\paragraph{Remark about novelty of our analysis:}
Compared to prior works, we are the first to propose a class of globally convergent ES methods to handle noisy estimates of the objective function $f(x)$. The extension of the proposed framework to handle constraints in a stochastic setting, as described by Eq.~\ref{extreme_2}, is a second contribution of this work. The convergence analysis in particular requires designing a new Lyapunov function (as given by Eq.~\ref{eq:Lyapunov}) that depends on both the function values and the level of noise on the objective function. Detailed proofs are provided in the appendix.

% !TEX root = main.tex

\vspace{-2mm}
\section{IMPLEMENTATION AND TEST CASES}
\vspace{-2mm}

We tested an adaptation of our Algorithm~\ref{alg:GESgc} that used the  guided search approach introduced in~\cite{maheswaranathan2018guided}. The Guided-Evolution Strategy (GES) defines a search distribution from a subspace spanned by a set of surrogate gradients. Importantly, this modification is also covered by the convergence guarantees derived in Theorem~\ref{th:1} (note that~\cite{maheswaranathan2018guided} did not provide convergence guarantees and was not applied to the constrained setting). We refer the reader to Section~\ref{sec:guided_search} in the appendix for further details. From here we denote our implementation of Algorithm \ref{alg:GESgc} as PCCES ({\bf{P}}rovably {\bf{C}}onvergent {\bf{C}}onstrained {\bf{ES}}). In what comes next, we present two different reinforcement learning applications. % In both applications, we optimize the parameters $x$ of a policy denoted by $\pi_x$ subject to a set of constraints $\Omega$ that depend on the policy.

\vspace{-2mm}
\paragraph{Constrained entropy maximization:}
We consider the standard formalization of reinforcement learning (RL) as a finite time Markov Decision Process (MDP). At each time step $t$, an RL agent receives a state $s_t$ based on which it selects an action $a_t$ using a policy denoted by $\pi$. The environment then provides a reward $r_t$ and a new state to the agent. We optimize the policy such that it learns to output the optimal sequence of actions that maximizes the cumulative reward over all steps. Formally, we consider a trajectory $\tau := (s_0, a_0, r_0, \ldots, s_T, a_T, r_T)$ as a sequence of state-action-reward triples which is distributed according to $\pi$. The goal of the RL agent is to maximize the objective $R(\tau) = \sum_{t=0}^T \theta^t r_t$,
where $\theta \in [0,1]$ is the discount factor. We consider stochastic policies $\pi_x(a|s)$~\footnote{In the following, we may omit the subscript $x$ for simplicity. However, the reader should understand a maximization over $\pi$ as maximizing over the parameters $x$ of the neural network.} which are parameterised by $x$, the parameters of a neural network. 
Then, we define the expected return as a function of trajectories generated by a policy, i.e., $\E_{\tau \sim \pi}[R(\tau)]$,
where the return $R(\tau)$ is the discounted sum of rewards from the trajectory $\tau$. The problem of finding the optimal policy $\pi^{*}$ is thus given by $
    \pi^{*} := \argmax_{\pi} \E_{\tau \sim \pi}[R(\tau)]$.
 
In this paper, we change the latter unconstrained optimization problem by adding constraints while maximizing the entropy of the learned policy. This has been shown to improve the exploration abilities of the agent and as a result yield higher return policies \citep{mnih2016asynchronous, ahmed2018understanding, haarnoja2018soft}. In evolution strategies, the iterations are episodic in nature and not over time steps. Thus, we define the entropy of a policy over complete trajectories. We define the entropy of a policy $\pi$ over a trajectory as the sum of the entropy over states in the trajectory, i.e. $H_{\pi}(\tau) =  \sum_{t=1}^{T} h_{\pi}(s_{t})$, where, $\tau$ is the trajectory and $h_{\pi}(s_{t})$ is the entropy of the policy distribution at time step $t$. We then define the constraint set $\Omega$ as constraints that determine an acceptable interval for the entropy $H_{\pi}(\tau)$.
%
% Constrained problem
We obtain the following constrained entropy maximization problem:
\begin{align}
    &\max_{\pi \in \Omega_1} \left[ \E_{\tau \sim \pi}[R(\tau)] + \mu \E_{\tau \sim \pi}[H_{\pi}(\tau)] \right], 
\label{eq:opt_pi_c}
\end{align}
where ${\small \Omega_1 = \{ \pi: h_l \leq \E_{\tau \sim \pi}[H_{\pi}(\tau)] \leq h_{u} \}}$,  $h_{u}$ and $h_{l}$ are fixed bounds for the entropy and $\mu > 0$ is weights the importance of the entropy term . 
\begin{figure*}[h]
    \centering
    \includegraphics[width=\linewidth]{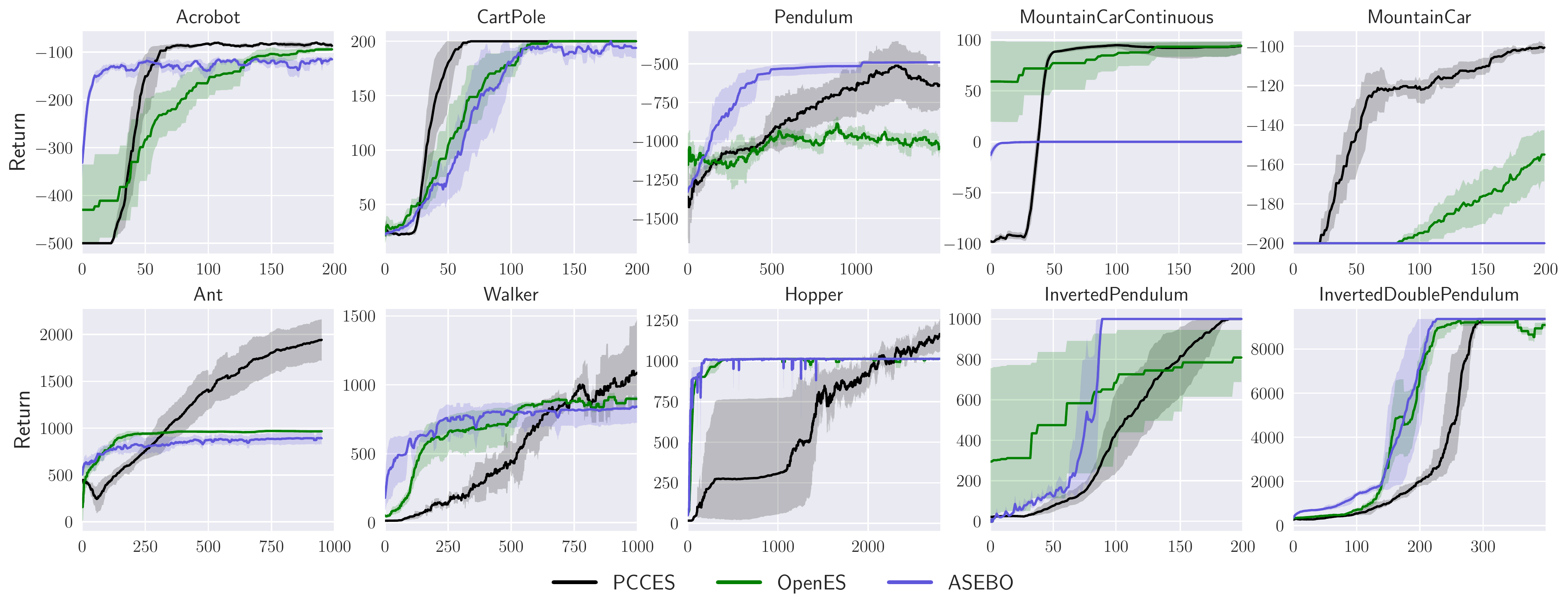}
    \vspace{-2mm}
    \caption{\small{PCCES compared with OpenES and ASEBO. We observe that even though the methods perform similarly in the early optimization phase (PCCES sometimes being slower), PCCES consistently achieves higher returns.}}
    \label{fig:max_ent_plots}
\end{figure*}            
\vspace{-1mm}
\paragraph{Policy optimization with constraints:} In this application, we optimize for policies which maximize reward while including non-relaxable conditions over the MDP. Constrained Markov Decision Processes (CMDP) \citep{Altman1999ConstrainedMD} is a framework for representing systems with such conditions. Similar to the standard MDP framework, we can then obtain the optimal policy by maximizing the return, over the set of policies which satisfy the constraints. We assume that we are given a set of constraint cost functions $g_{\pi}^{1}(\tau), \dots, g_{\pi}^{r}(\tau)$ which depend on the policy $\pi$ used in a specific application. Hence
{\small
$
\Omega_2 := \{\pi: \E_{\tau \sim \pi} [g_{\pi}^{i}(\tau)] \leq t_{i}, i = 1,\dots,r \},
$}
where each $ \{ t_i \}_{i=1}^{r}$ are chosen threshold values.
Then, for $\mu$ is a penalty parameter, we target to solve the optimization problem
\vspace{-2mm}
\begin{equation}
\label{eq:barrier_const}
{\small
    \max_{\pi \in \Omega_2}  [\E_{\tau \sim \pi}[R(\tau)] + \mu \E_{\tau \sim \pi}  [\sum_{i=1}^{r} g_{\pi}^{i}(\tau)]].}
\end{equation}

% !TEX root = main.tex

\vspace{-2mm}
\section{EXPERIMENTAL RESULTS}
\label{sec:experiments}
\vspace{-2mm}
\begin{figure*}[!h]
    \centering
    \includegraphics[ width=\linewidth]{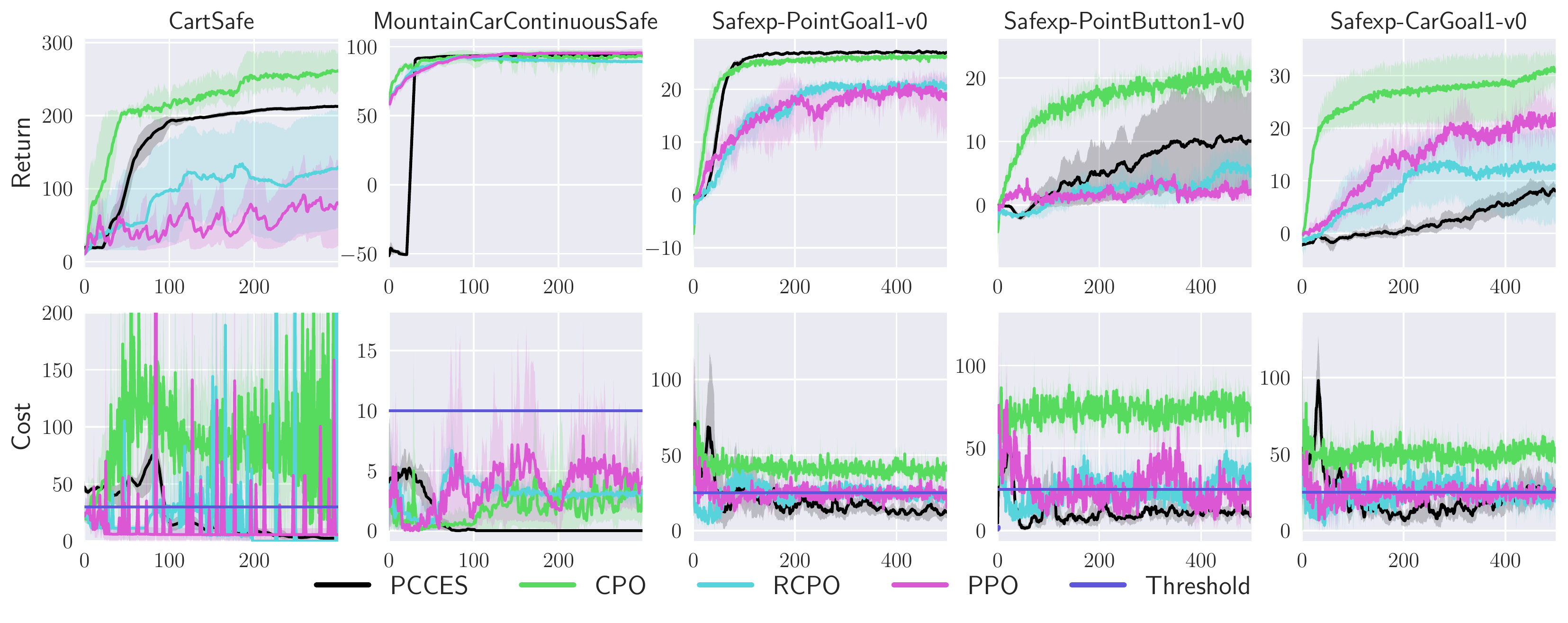}
    \vspace{-2mm}
    \caption{\small{PCCES compared with CPO, RCPO and PPO. The horizontal blue line is the threshold under which the cost should stay. Note that due to the stochastic estimates, PCCES slightly violates the constraints at the beginning but it consistently returns a final feasible solution. The amount of violation can of course be controlled by lowering the stochaticity of the environment. Importantly, all the other approaches we benchmarked are not below the threshold during the entire training phase. This is especially important at the end of training where the returned solution does not satisfy the desired constraints. }}
    \label{fig:const_plots}
\end{figure*}            

\textbf{Constrained entropy maximization:}
We first evaluate PCCES on 5 control tasks  available in OpenAI Gym \citep{brockman2016openai} and 5 control tasks from PyBullet \citep{coumans2019}.  
In all of these tasks, the goal is to maximize the accumulated reward over a finite number of steps. We compare the performance of PCCES with OpenES~\citep{salimans2017evolution} and  ASEBO~\citep{asebo}. We chose to compare against OpenES as it is one of the most popular algorithms in RL while ASEBO is specifically designed to address the exploration-exploitation trade-off. We note that~\cite{salimans2017evolution} showed that OpenES performs as well as its model free counterparts such as TRPO \citep{schulman2017proximal} and A3C \citep{mnih2016asynchronous} over a large number of benchmarks. 

We optimize a policy parameterised by a two-layer network with 10 and 64 hidden units for tasks in OpenAI Gym and PyBullet respectively. At every iteration, 40 points are sampled for the OpenAI Gym tasks and 120 points are sampled for the PyBullet tasks. For PCCES, policy gradients at $m$ previous timesteps are used as surrogate gradients. Update directions in the previous timesteps can also be used as surrogate gradients. We use $m=20$ and we do not update our policy for the first 20 iterations. We conduct ten runs with varying random seeds for each environment. The policy is evaluated during training at every update and we store the return as the average of the last ten evaluations.  We report the average return with standard deviation over the different runs during training in Figure \ref{fig:max_ent_plots}. We also report the entropy during training compared against the entropy constraint in Appendix, see Figure \ref{fig:entropy_const_exp}. The remaining training curves are in Figures (\ref{fig:acrobot_1}-\ref{fig:mcc_1}) in the appendix. For details regarding other hyperparameters, we refer the reader to Section~\ref{sec:exp_details} in the appendix.\\

%\vspace{-4mm}
\textbf{Policy Optimization with Constraints:} In this set of experiments, we evaluate our algorithm on five control tasks with constraints on the state space. ``CartSafe" and ``MountainCarContinuousSafe" are modifications of the ``CartPole" and ``MountainCarContinuous" environments from OpenAI Gym and PointGoal, PointButton and CarGoal are from the safety-gym \cite{Ray2019}, with constraints that penalize visiting restricted states and the goal is to maximize reward while keeping the constraint penalty below a threshold. We compare our performance with Constrained Policy Optimization (CPO) \citep{achiam2017constrained}, Reward Constrained Policy Optimization (RCPO) \cite{tessler:18} and Proximal Policy Optimization \cite{schulman2017proximal} with Lagrange constraints in both the environments.
% We evaluate two versions of PCCES, one with the constraint penalty added to reward (PCCES-P) and another without such a penalty (PCCES) (i.e., meaning that $\mu=0$ in Eq.~\ref{eq:barrier_const}). 
The cost threshold for CartSafe is 30, MountainCarContinuousSafe is 10 and for the safety-gym environments is 25. We use a policy parameterized by two-layer  networks with tanh units. We train our algorithm for 300 iterations and train CPO, RCPO, PPO for 300 epochs (30000 timesteps in each epoch) for CartSafe and MoutainCarContinuousSafe. We increase the iterations/epochs to 500 for safety-gym tasks. PCCES does only one update per iteration, while the rest do multiple updates per epoch. This suggests a better scaling of PCCES with less wall-clock time per iteration. We average the runs over 10 different random seeds and plot the cost and return. 

\textbf{Analysis:} Figure \ref{fig:max_ent_plots} (and Figures \ref{fig:acrobot_1}-\ref{fig:mcc_1} in appendix) shows that PCCES always finds a better or equivalent solution and also does not stagnate or diverge in the tasks we tried. The sufficient increase condition makes PCCES more robust against divergence and entropy maximization pushes the algorithm to explore the state space, as a result PCCES eventually reaches better returns for a majority of tasks. Ant is a challenging task due to noisy estimates of its return and the need for exploration. The importance of entropy maximization and the sufficient increase condition can be most clearly seen in the Ant task, for which the performance of OpenES and ASEBO stagnates. Similarly, OpenES and ASEBO performance stagnates in Walker and Hopper, while PCCES consistently keeps improving. 
Figure \ref{fig:entropy_const_exp} (appendix) shows that PCCES consistently ensures the entropy constraints are also satisfied.

Figure \ref{fig:const_plots} shows the performance of PCCES compared to various baselines (CPO, RCPO, PPO) for the constrained tasks.
PCCES consistently returns a solution which has a cost lower than the threshold and improves the episodic return for all the environments tested. CPO on the other hand does not keep the cost below threshold, while RCPO and PPO keep the cost low but achieve low returns compared to PCCES for certain environments (CartSafe, PointGoal, PointButton). 
% 

% !TEX root = main.tex

\section{CONCLUSION}

We proposed a class of evolutionary method to solve stochastic constrained problems, with applications to reinforcement learning. One feature that distinguishes our approach from prior work is its global convergence guarantee for stochastic setting. We note that our proof technique does not exploit any specific information about the constraints being related to the entropy or the safety of the policy. Our algorithm could therefore be applied to different types of constraints and also problems outside the area of reinforcement learning. Empirically, we have seen that PCCES consistently achieves higher returns compared to other baselines but might require more iterations to do so. One potential direction to address this problem would be to anneal the constraints. Another future direction to pursue would be to benchmark PCCES in a distributed setting, especially given that prior work by~\cite{salimans2017evolution} demonstrated that the strength of ES lies in its ability to scale over large clusters with less communication between actors. Other directions of interest would be to develop convergence rates as well as second-order guarantees~\citep{gratton2016second, lucchi2021second}. Finally, our approach could also be extended to other problems in machine learning, such as min-max optimization~\citep{anagnostidis2021direct}.

\section{ACKNOWLEDGEMENT}

Part of this work was performed while Aurelien Lucchi and Vihang Patil were at ETH Z\"urich.

\bibliography{reference}
%\bibliographystyle{plain}

%%%%%%%%%%%%%%%%%%%%%%%%%%%%%%%%%%%
%%%%%% SUPPLEMENT (OPTIONAL) %%%%%%
%%%%%%%%%%%%%%%%%%%%%%%%%%%%%%%%%%%

\clearpage
\appendix

\thispagestyle{empty}

% For one-column format, uncomment the following:
\onecolumn \makesupplementtitle
% For two-column format, uncomment the following:
%\twocolumn[ \makesupplementtitle ]

% !TEX root = main.tex

%%%%%%%%%%%%%%%%%%%%%%%%%%%%%%%%%%%%%%%%%%%%%%%%%%%%%%%%%%%%
%%%%%%%%%%%%%%%%%%%%%%%%%%%%%%%%%%%%%%%%%%%%%%%%%%%%%%%%%%%%

\section{Main analysis}

\subsection{Existence of a converging subsequence}

\begin{lemma} \label{lm:1}Let Assumptions \ref{ass:f}  and   \ref{ass:ProbAccEstimates:obj}  hold. Then, for every iteration $k$, one has 
$$
\mathbb{E}\left(\mathbf{1}_{\overline{T^f_k}} \left|F_k^0 - f(X_k) \right|  | \mathcal{F}_{k-1} \right)  \le (1-p_k^*)^{1/2} \varepsilon_v \Sigma_k^2
$$
and
$$
\mathbb{E}\left(\mathbf{1}_{\overline{T^f_k}} \left|F_k^1 - f(X_k+S_k) \right|  | \mathcal{F}_{k-1} \right)  \le (1-p_k^*)^{1/2} \varepsilon_v \Sigma_k^2
$$
where $\overline{T^f_k}$ denotes the event where the objective function estimates $(F_k^0, F_k^1)$ are inaccurate. $\mathbf{1}_{\overline{T^f_k}}$  is the indicator function of the event set $\overline{T^f_k}$.
\end{lemma}
\begin{proof}
The proof is the same as in Lemma 1~\cite{Sto_Mads_2019}.
\end{proof}
 \begin{lemma} \label{lm:2}Let Assumptions \ref{ass:f}  and   \ref{ass:ProbAccEstimates:obj}  hold. Then for every iteration $k$, one has 
$$
\mathbb{E}\left(\mathbf{1}_{T^f_k} \left(\Phi_{k+1} - \Phi_k\right) | \mathcal{F}_{k-1} \right)  \le  - (1 - \nu)(1- \gamma^{-2})p_k^*\Sigma_k^2,
$$
 where $\mathbf{1}_{T^f_k}$  is the indicator function of the event set $T^f_k$.
\end{lemma}
\begin{proof}
Consider a realization of a given iteration $k$ of Algorithm \ref{alg:GESgc} for which the objective function estimates $ \{f_k^0, f_k^1\}_k$ are $\varepsilon_f$-accurate. Then,
if the iteration is unsuccessful $x_{k+1} = x_k$ and $\sigma_{k+1} = \gamma^{-1} \sigma_k$, this leads to 
	\begin{equation} \label{eq:DeltaPhiUnsuccessful}
		\phi_{k+1} - \phi_{k} = (1 - \nu)(\gamma^{-2}  - 1)\sigma_k^2:=b_1 <0.
	\end{equation}
	Otherwise, if the iteration is successful, then one has $x_{k+1} = x_k + \sigma_k d_k$, where  $\sigma_{k+1} \leq \gamma \sigma_k$. One thus has:
	\begin{eqnarray*} 
		\phi_{k+1} - \phi_k &\leq& \nu \left( f(x_{k+1}) - f(x_k) \right) + (1 - \nu)(\gamma^2 - 1)\sigma_k^2 \\
		& \le & \left(\nu (2\varepsilon_f - \frac{\kappa}{2})+ (1 - \nu)(\gamma^2 - 1) \right) \sigma_k^2\\
		& \le & \left(- \nu \frac{\kappa}{4}+ (1 - \nu)(\gamma^2 - 1) \right) \sigma_k^2,
	\end{eqnarray*}
	where the last inequality is obtained using  the fact that $\varepsilon_f\le \frac{\kappa}{4}$.
	Assuming that  $ \frac{\nu}{1-\nu} \ge \frac{8(\gamma^2 - 1)}{\kappa}$, one deduces that 
	\begin{eqnarray*} 
		\phi_{k+1} - \phi_k & \le & - \nu \frac{\kappa}{8} \sigma_k^2:=b_2<0.
	\end{eqnarray*}
		We note also that, as $\frac{\nu}{1 - \nu} \ge   \frac{8(\gamma^2 - 1)}{\kappa\gamma^2}$, one has   $b_2\le b_1$. Hence, one deduces that for any iteration $k$ conditioned by this case, one has 
		\begin{eqnarray*} \label{eq:b1}
		\phi_{k+1} - \phi_k & \le & b_1.
	\end{eqnarray*}
	Hence, since the event  $T^f_k$ occurs with a probability $p_k^*$, one gets
	$$
\mathbb{E}\left(\mathbf{1}_{T^f_k} \left(\Phi_{k+1} - \Phi_k\right) | \mathcal{F}_{k-1} \right)  \le - (1 - \nu)(1- \gamma^{-2})p_k^*\Sigma_k^2.
$$
\end{proof}
\begin{lemma} \label{lm:3}Let Assumption \ref{ass:f}  hold. Suppose that Assumption \ref{ass:ProbAccEstimates:obj}. Then, for every iteration $k$, one has 
$$
\mathbb{E}\left(\mathbf{1}_{\overline{T^f_k}}(\Phi_{k+1} - \Phi_k)  | \mathcal{F}_{k-1}\right)	  \le 2 \nu (1-p_k^*)^{1/2} \varepsilon_v \Sigma_k^2.
$$
\end{lemma}
\begin{proof}
Conditioned by the event $\mathcal{F}_{k-1}$, assuming that the objective function estimates $ \{F_k^0, F_k^1\}_k$ are inaccurate. One has, if the iteration $k$ is successful,  $X_{k+1} = X_k + \Sigma_k D_k$, where  $\Sigma_{k+1} \leq \gamma \Sigma_k$. Hence, using Lemma \ref{lm:1}, we get
		\begin{eqnarray*} 
		\mathbb{E}\left(\mathbf{1}_{\overline{T^f_k}}(\Phi_{k+1} - \Phi_k)  | \mathcal{F}_{k-1}\right)& \le & \nu \mathbb{E}\left( \mathbf{1}_{\overline{T^f_k}}\left(f(X_{k+1}) - f(X_k)\right) | \mathcal{F}_{k-1}\ \right) + (1 - \nu)(1-p_k^*)(\gamma^2 - 1)\Sigma_k^2 \\
		& \le & \nu \left(  2 (1-p_k^*)^{1/2} \varepsilon_v - \frac{c(1-p_k^*)}{2} \right)\Sigma_k^2 + (1 - \nu)(1-p_k^*)(\gamma^2 - 1)\Sigma_k^2 
	\end{eqnarray*}
	Assuming that  $ \frac{\nu}{1-\nu} \ge \frac{4(\gamma^2 - 1)}{c}$,  one gets $-\nu (1-p_k^*) \frac{c}{2} + (1 - \nu)(1-p_k^*)(\gamma^2 - 1) \le 0$, hence
	\begin{eqnarray*} 
	\mathbb{E}\left(\mathbf{1}_{\overline{T^f_k}}(\Phi_{k+1} - \Phi_k)  | \mathcal{F}_{k-1}\right)	& \le & 2 \nu (1-p_k^*)^{1/2} \varepsilon_v \Sigma_k^2.
	\end{eqnarray*}
If the iteration is unsuccessful $X_{k+1} = X_k$ and $\Sigma_{k+1} = \gamma^{-1} \Sigma_k$,
 this leads to 
	\begin{equation*}
	\mathbb{E}\left(\mathbf{1}_{\overline{T^f_k}}(\Phi_{k+1} - \Phi_k) | \mathcal{F}_{k-1}\right) = (1-p_k) (1 - \nu)(\gamma^{-2}  - 1)\Sigma_k^2 \le  2 \nu (1-p_k^*)^{1/2} \varepsilon_v \Sigma_k^2.
	\end{equation*}
	Hence, in both cases, one has
	\begin{equation*} \label{eq:DeltaPhiUnsuccessful_2}
		\mathbb{E}\left(\mathbf{1}_{\overline{T^f_k}}(\Phi_{k+1} - \Phi_k)  | \mathcal{F}_{k-1}\right)	 = (1 - \nu)(\gamma^{-2}  - 1)(1-p_k^*) \Sigma_k^2 <  2 \nu (1-p_k^*)^{1/2} \varepsilon_v \Sigma_k^2.
	\end{equation*}
\end{proof}

\DecreaseOnPhi*
\begin{proof}
The  proof is inspired from what is done in \cite{EBergou_YDiouane_VKungurtsev_CWRoyer_2018a,ChenMenickellyScheinberg17,BlanchetCartisMenickellyScheinberg19,Sto_Mads_2019}. 

Putting the results of the two Lemmas \ref{lm:2} and \ref{lm:3} together, we obtain the following
	\begin{equation*} %\label{eq:MainGoal}
		\mathbb{E}\left(\Phi_{k+1} - \Phi_k | \mathcal{F}_{k-1}\right)\leq - (1 - \nu)p_k^*(1- \gamma^{-2})\Sigma_k^2 +  2 \nu (1-p_k^*)^{1/2} \varepsilon_v \Sigma_k^2
	\end{equation*}
	Hence, assuming that $\frac{p_k^*}{(1-p_k^*)^{1/2}} \ge  \frac{p}{(1-p)^{1/2}} \ge \frac{4\nu \varepsilon_v}{(1-\nu)(1-\gamma^{-2})} $, it reduces to 
			\begin{equation*} %\label{eq:MainGoal}
		\mathbb{E}\left(\Phi_{k+1} - \Phi_k | \mathcal{F}_{k-1}\right)\leq -  \frac{1}{2}(1 - \nu) p_k^* (1- \gamma^{-2}) \Sigma_k^2\le - \alpha \Sigma_k^2
	\end{equation*}
	where $\alpha =    \frac{1}{2}(1 - \nu)p(1- \gamma^{-2})\Sigma_k^2 $.
%\end{enumerate}

\end{proof}

\StepSizeToZeroAlmostSurely*
\begin{proof}
Indeed, by taking expectation on the result from Theorem \ref{thm:DecreaseOnPhi}, one gets 
\begin{equation*} 
	\mathbb{E}\left(\Phi_{n+1} - \Phi_0 \right) = \sum^{n}_{k=0}  \mathbb{E}\left(\Phi_{k+1} - \Phi_k \right) =  \sum^{n}_{k=0} \mathbb{E}\left( \mathbb{E}\left(\Phi_{k+1} - \Phi_k | \mathcal{F}_{k-1}\right) \right)  \le - \alpha  \mathbb{E}\left(  \sum^{n}_{k=0}  \Sigma_k^2 \right).
\end{equation*}
Since $\Phi_{n+1} \geq \nu f_{\mathrm{low}}$, one deduces that by taking $n \to \infty$
\begin{equation*}
	 \mathbb{E}\left(  \sum^{+\infty}_{k=0}  \Sigma_k^2 \right)  < \infty.
\end{equation*}
Thus, we conclude that the probability of the random variable $\Sigma_k$ to converge to zero is one.	
Moreover, assuming the boundedness  of the  sequence $\{X_k\}$, one deduces the existence of random vector $X_*$ and a subsequence $\mathcal{K}\subset \mathbb{N}$ such that $\{\Sigma_k\}_{k \in \mathcal{K}}$ goes to zero almost surely and  $\{X_k\}_{k\in  \mathcal{K}}$ converges almost surely to $X_*$.
\end{proof}

\subsection{Optimality condition for constrained problems}
\label{sec:optimality_condition}

We now turn to deriving a main global convergence result.

\paragraph{Review of required definitions}

In what comes next, we introduce the formal definition of a hypertangent cone which will be required to state our main convergence theorem. We will denote by $B(x;\epsilon)$ the closed ball formed by all points at a distance of no more than~$\epsilon$ to~$x$.
\begin{definition}
A vector $d\in \R^n$ is said to be a hypertangent vector to the set $\Omega \subseteq \R^n$
at the point $x$ in $\Omega$ if there exists a scalar $\epsilon>0$ such that
\begin{align*}
&y+tw\in\Omega,\quad\forall y\in\Omega\cap B(x;\epsilon),\quad w\in B(d;\epsilon), \\
&\qquad\text{and}\quad 0<t<\epsilon.
\end{align*}
\end{definition}

The hypertangent cone to $ \Omega$ at $x$, denoted by~$T_{\Omega}^{\Hc}(x)$, is the set of all hypertangent vectors
to $\Omega$ at $x$.
Then, the Clarke tangent cone to $\Omega$ at $x$ (denoted by~$T_{\Omega}^{\cl}(x)$)
can be defined as the closure of the hypertangent cone~$T_{\Omega}^{\Hc}(x)$.

\begin{definition}
A vector $d\in \R^n$ is said to be a Clarke tangent vector to the set $\Omega \subseteq \R^n$
at the point $x$ in the closure of $\Omega$ if for every sequence $\{x_k\}$ of elements of
$\Omega$ that converges to $x$ and for every sequence of positive real numbers $\{t_k\}$
converging to zero, there exists a sequence of vectors $\{d_k\}$ converging to $d$ such
that $x_k+t_kd_k\in\Omega$.
\end{definition}

\paragraph{Auxiliary result}
We state an auxiliary result from the literature that will be useful for the analysis (see Theorem 5.3.1~\cite{RDurret_book_2010} and Exercise 5.3.1~\cite{RDurret_book_2010}).
\begin{lemma}
\label{supermartingal:lm}
Assume that, for all $k$, $G_k$ is a supermartingale with respect to $\mathcal{F}_{k-1}$ (a $\sigma$-algebra generated by $G_0,\ldots, G_{k-1}$. Assume further that there exists $M>0$ such that $|G_{k} -G_{k-1}| \le M <\infty$, for all $k$. Consider the random events $C=\{\lim_{k\to \infty} G_k ~\mbox{exists and is finite}\}$ and $D=\{\lim \sup_{k\to \infty} G_k = \infty\}$. Then $\mathbf{P}(C\cap D)=1$.
\end{lemma}

\paragraph{Main convergence result}

%\begin{mybox}{gray}
\begin{theorem}[Formal version of Theorem~\ref{th:1}]
\label{th:12}
Assume Assumption \ref{ass:AccEstimates:cons} and the working assumptions of Corollary \ref{corol:StepSizeToZeroAlmostSurely} hold.
Then, there exists an almost surely event $A$ such that for all $\omega \in A$, $X_*(\omega) \in \Omega $ is a limit point of the subsequence of iterates $\{ X_k(\omega) \}_K$ and $\lim_{k \in K} \Sigma_k(\omega) = 0$.
In this case, if $d \in T_{\Omega}^{\cl}(X_*(\omega))$ is a limit point associated with $\{ D_k(\omega) \}_K$, then $\nabla f\left(X_*(\omega)\right)^{\top}d\ge 0.$
\end{theorem}
%\end{mybox}
\begin{proof}
From Corollary \ref{corol:StepSizeToZeroAlmostSurely} and Assumption \ref{ass:AccEstimates:cons}, it follows that the event $$A=\left\{\omega \in \mathcal{E}: \exists \mathcal{K}\subset \mathbb{N}~\mbox{such that $\{\Sigma_k(\omega)\}_{k \in \mathcal{K}} \to 0$ and  $\{X_k(\omega)\}_{k\in  \mathcal{K}} \to X_*(\omega)$} \right\} \cap \left \{\cap_{k=0}^{\infty} T^c_k \right\}$$
happens  almost surely. Now, consider $\omega \in A$ and let $x_*= X_*(\omega) \in \Omega$, $x_k= X_k(\omega) $, $\sigma_k= \Sigma_k(\omega) $ and $d_k= D_k(\omega)$. 
Let  $d \in T_{\Omega}^{\cl}(X_*(\omega))$ be a limit point associated with $\{ d_k \}_K$. Then, conditioned by the event $A$, one has for $k \in \mathcal{K}$ sufficiently large $x_k+\sigma_k d_k \in \Omega$.

Let $W_k=\sum^{k}_{i=1} \left(2\mathbf{1}_{T^f_i} -1\right)$ where $T^f_i$ is given in Definition \ref{defi:probaccfun}, recall that by Assumption \ref{ass:ProbAccEstimates:obj}, $p_i^*= \mathbb{P}(T^f_i | \mathcal{F}_{i-1}) \ge \frac{1}{2}$. We start by showing that $\{W_k\}$ is a submartingale:
$$
\mathbb{E}(W_k| \mathcal{F}_{k-1}) = W_{k-1} + 2  \mathbb{P}(T^f_k | \mathcal{F}_{k-1}) -1 \ge W_{k-1}.
$$
Note that $|W_{k+1}- W_k| =1$, hence the event $\{ \lim_{k\to \infty}W_k ~\mbox{exists and is finite} \}$ has a probability zero. Thus by Lemma \ref{supermartingal:lm}, one deduces that $\mathbb{P}(\lim \sup_{k} W_k =\infty)=1$.

Conditioned by the event $A$, suppose that there exists $\epsilon>0$,   $\nabla f\left(x_*\right)^{\top}d\le -2\epsilon$. Hence, there exists $k_1$ such that for $k \in \mathcal{K}$ and $k\ge k_1$, one has $\frac{f(x_k+\sigma_k d_k) - f(x_k)}{\sigma_k} \le -\epsilon$ and $x_k+\sigma_k d_k \in \Omega$. By Corollary \ref{corol:StepSizeToZeroAlmostSurely}, conditioned by A, one has $\sigma_k \to 0$ when $k$ goes to $\infty$. Thus, there exists $k_2$ such that for $k \in \mathcal{K}$ and $k\ge k_2$, one has
$$
\sigma_k \le b_{\epsilon}:=\frac{2 \epsilon}{\kappa+4\varepsilon_f}.
$$
For any $k \in \mathcal{K}$ such that $k \ge k_0=\max\{ k_1,k_2\}$, we note that since $x_k+\sigma_k d_k \in \Omega$ and Assumption \ref{ass:AccEstimates:cons} holds, one deduces that $c_k^0 - \varepsilon_c \sigma_k \le 0$ and $c_k^1 - \varepsilon_c \sigma_k \le 0$, meaning  that $\tilde f^{\trial}_k =f^{1}_k$. Two cases then occur. First if $\mathbf{1}_{T^f_k} =1$, then 
$$
\tilde f^{\trial}_k  - f_{k} \le 2 \varepsilon_f \sigma_k^2  - \epsilon \sigma_k \le - \frac{\kappa}{2} \sigma_k^2,
$$
Hence, the iteration $k$ of Algorithm \ref{alg:GESgc} is successful and the stepsize $\sigma_k$ is updated as $\sigma_{k+1}= \gamma \sigma_k$.

Let now $B_k$  be the random variable whose realization is $b_k = \log_\gamma\left( \frac{\sigma_k}{b_{\epsilon}}\right)$. Clearly, if $\mathbf{1}_{T^f_k} =1$, one has 
$b_{k+1} =b_k +1$. Otherwise, if $\mathbf{1}_{T^f_k} =0$, $b_{k+1}\ge b_{k} -1$ since $\sigma_{k+1} \ge \gamma^{-1} \sigma_k$ always holds.
Hence, $B_{k}- B_{k_0} \ge W_{k} -W_{k_0}$, and from $\mathbb{P}(\lim \sup_{k} W_k =\infty)=1$, one obtains $\mathbb{P}(\lim \sup_{k} B_k =\infty)=1$. This leads to a contradiction with the fact that $B_k<0$ for any $k \in \mathcal{K}$ such that $k \ge k_0$. 
\end{proof}

%%%%%%%%%%%%%%%%%%%%%%%%%%%%%%%%%%%%%%%%%%%%%%%%%%%%%%%%%%%%%%%%%%%%%%%%%%%%%%%%
%%%%%%%%%%%%%%%%%%%%%%%%%%%%%%%%%%%%%%%%%%%%%%%%%%%%%%%%%%%%%%%%%%%%%%%%%%%%%%%%

\newpage
\section{Guided-evolution strategy}
\label{sec:guided_search}

As an efficient implementation of  Algorithm~\ref{alg:GESgc}, we tested the GES approach  introduced in~\cite{maheswaranathan2018guided}. The GES technique defines a search distribution from a subspace spanned by a set of surrogate gradients\footnote{A surrogate gradient is defined as an biased or corrupted gradient, which has correlation with the true gradient.} denoted by $\S_k$. At each iteration $k$, the set $\S_k=[G_{k-1-m}, \ldots, G_{k-1}] \in \real
^{n\times m}$ consists of the last $m$ surrogate gradients computed from iterations $\{k-1-m, \dots k-1 \}$. The set $\S_k$ is used to compute an orthogonal basis $\Um_k$ of the subspace formed by the vectors in $\S_k$. This is done using a QR decomposition as specified in~\cite{maheswaranathan2018guided}.

Samples are then drawn around the mean vector $x_k$ according to the distribution $\N(0, (\sigma_{k}^{\es})^2 C_k)$, where the covariance matrix is given by $C_k = \frac{\alpha}{n} I_{n} + \frac{1 - \alpha}{m} \Um_k \Um_k^\top$,
where $I_{n}$ is the identity matrix and $\alpha \in [0,1]$ is a hyperparameter that trade-offs the influence of the smaller subspace $\R^m$ over the entire space $\R^n$. A small value of $\alpha$ enforces the search to be conducted in the smaller subspace while larger values give the smaller subspace less importance. %\cite{maheswaranathan2018guided} suggest to use $\alpha = 0.5$. 
In practice the directions $d^i_k$ can efficiently sampled as follows, 
\begin{equation}
    d^i_k = \sigma_{k}^{\es}\sqrt{\frac{\alpha}{n}}d + \sigma_{k}^{\es}\sqrt{\frac{1 - \alpha}{m}} \Um_k d'
    \label{eq:practiceges}
\end{equation}
where, $d \sim \N(0, I_{n})$, $d' \sim \N(0, I_{m})$ and $\sigma_{k}^{\es}$ is the standard deviation of the distribution from which the direction's are sampled.

The surrogate gradient can be computed in various manners. Update directions in previous iterations can also be used to compute surrogate gradients. We compute the surrogate gradients required to compute $\Um_k$  using an Actor-Critic \citep{sutton2018, mnih2016asynchronous} formulation of policy gradient, where every member of the population $i$ computes an approximate gradient as
% comment: Sum over the population to get the approximate policy gradient
\begin{equation}
G_{\pi_{y^i_k}} := A_{\pi_{y^i_k}}(s_{t}, a_{t})\nabla \log \pi_{y^i_k}(a_{t}|s_{t})
\end{equation}

where $\pi_{y^i_k}$ is parameterized by $y^i_k$, $A_{\pi_{y^i_k}}(s_{t}, a_{t}) = R_t - V_{v}(s_{t})$ is the advantage function and, $s_{t}, a_{t}$ are the state and actions sampled at step $t$, $R_t$ is the return from step $t$ to the last time step $T$ and $V_{v}$ is the value function parameterized by parameters $v$. Then, the surrogate gradient is averaged to obtain the surrogate gradient for iteration k
\begin{equation}
G_{k} = \frac{1}{\lambda} \sum_{i=1}^{\lambda} G_{\pi_{y^i_k}}
\end{equation}
where, $\lambda$ is the size of the population. 

Before we start optimizing the policy, we compute $m$ surrogate gradients. We note that the guided-search strategy requires storing $m$ surrogate gradients in memory. If we already have $m$ surrogate gradients, then we discard the oldest surrogate gradient and replace it by the newer one.  Once we have enough surrogate gradients, we sample points around the $x_{k}$ using mirrored sampling \citep{salimans2017evolution}. These sample points are then evaluated on the environment, to compute the return and the entropy. Let $\lambda'=\lambda/2$, for a given $i=1,..., \lambda'$, we compute the estimation of the objective function for each of the sample point $y^i_k:=x_k + d^i_k$ as
\begin{equation*}
f_{k}^i \; = \; - R_{\tau \sim \pi_{y^i_k}}(\tau) - \mu \sum_{c=1}^{r} g_{\tau \sim \pi_{y^i_k}}^{c}(\tau)
\end{equation*}
and its mirror point $y^{i+\lambda'}_k:=x_k - d^i_k$ as
\begin{equation*}
f_{k}^{i+\lambda'} \; = \; - R_{\tau \sim \pi_{y^{i+\lambda'}_k}}(\tau) -  \mu \sum_{c=1}^{r} g_{\tau \sim \pi_{ y^{i+\lambda'}_k } }^{c}(\tau)
\end{equation*}
We then obtain the trial point (Eq.~\ref{eq:trial_point} in Algorithm~\ref{alg:GESgc}) using the following update rule, 
\begin{equation}
    x^{\trial}_{k+1} = x_k - \sigma_k\frac{\beta}{\sigma_{k}^{\es}\lambda} \sum_{i=1}^{\lambda'} \left(f_{k+1}^i -  f_{k+1}^{i+\lambda'}\right) d_{k}^i
    \label{eq:ges_parameter_update}
\end{equation}
where  $d^i_k$ is the the direction used to obtain $f_{k+1}^i$, $\sigma_k$ is the stepsize, and $\beta$ is a hyperparameter used for scaling.

Further, we compute the barrier function at the new trial point $x^{\trial}_{k+1}$ as, in the case of entropy maximization by
\begin{equation*}
\tilde f_{k+1}^{\trial} \; = \; \left\{\begin{array}{ll}-R(\pi_{x^{\trial}_{k+1}}) - \mu H(\pi_{x^{\trial}_{k+1}}) & \textrm{if } ~h_l \leq H(\pi_{x^{\trial}_{k+1}}) \leq h_u \\ +\infty & \textrm{otherwise.}
\end{array}\right.
\end{equation*}
Or, in the case of constrained policy optimization by
\begin{equation*}
\tilde f_{k+1}^{\trial} \; = \; \left\{\begin{array}{ll}-R_{\tau \sim \pi_{x^{\trial}_{k+1}}}(\tau) - \mu   \sum_{c=1}^{r} g^{c}_{\tau \sim \pi_{x^{\trial}_{k+1}}}(\tau)   & \textrm{if } [ g_{\tau \sim \pi_{x^{\trial}_{k+1}}}(\tau) ]_{c}\leq t_{c}, c = 1,\dots,r\\ +\infty & \textrm{otherwise.}
\end{array}\right.
\end{equation*}

For both cases, we accept the trial point (i.e. $x_{k+1}=x^{\trial}_{k+1}$ and $f_{k+1}=\tilde f_{k+1}^{\trial}$) if the following condition is satisfied,
\begin{equation*}
\tilde f_{k+1}^{\trial} \leq f_{k} - \frac{\kappa}{2}\sigma_k^2
\end{equation*}
where $\kappa > 0$ is a hyperparameter. We increase the $\sigma_k$ if the iteration is successful and decrease it if it is unsuccessful. 

%%%%%%%%%%%%%%%%%%%%%%%%%%%%%%%%%%%%%%%%%%%%%%%%%%%%%%%%%%%%%%%%%%%%%%%%%%%%%%%%
%%%%%%%%%%%%%%%%%%%%%%%%%%%%%%%%%%%%%%%%%%%%%%%%%%%%%%%%%%%%%%%%%%%%%%%%%%%%%%%%

\newpage
\section{Additional Experimental Results}
\label{sec:add_exp}
% Acrobot Plots

\begin{figure}[htp]

\centering
\includegraphics[width=\textwidth]{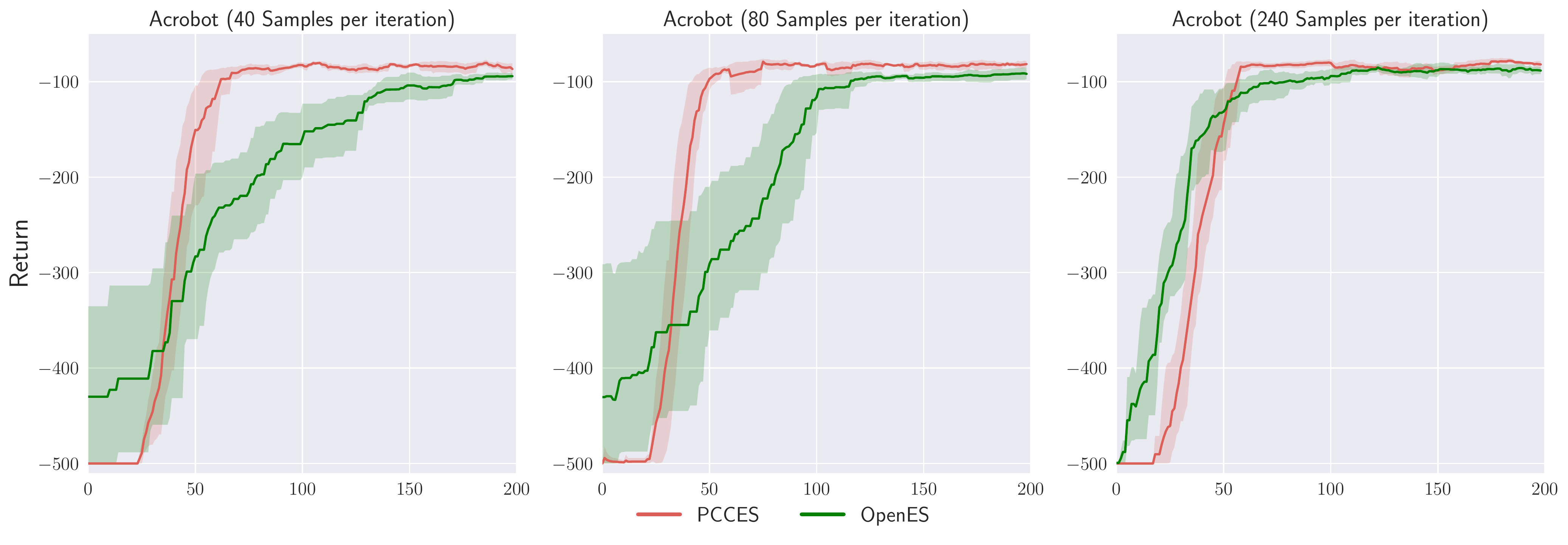}

\caption{Acrobot-v1: PCCES and OpenES runs with 40, 80 and 240 samples per iteration}
\label{fig:acrobot_1}

\end{figure}

% CartPole Plots
\begin{figure}[htp]

\centering
\includegraphics[width=\textwidth]{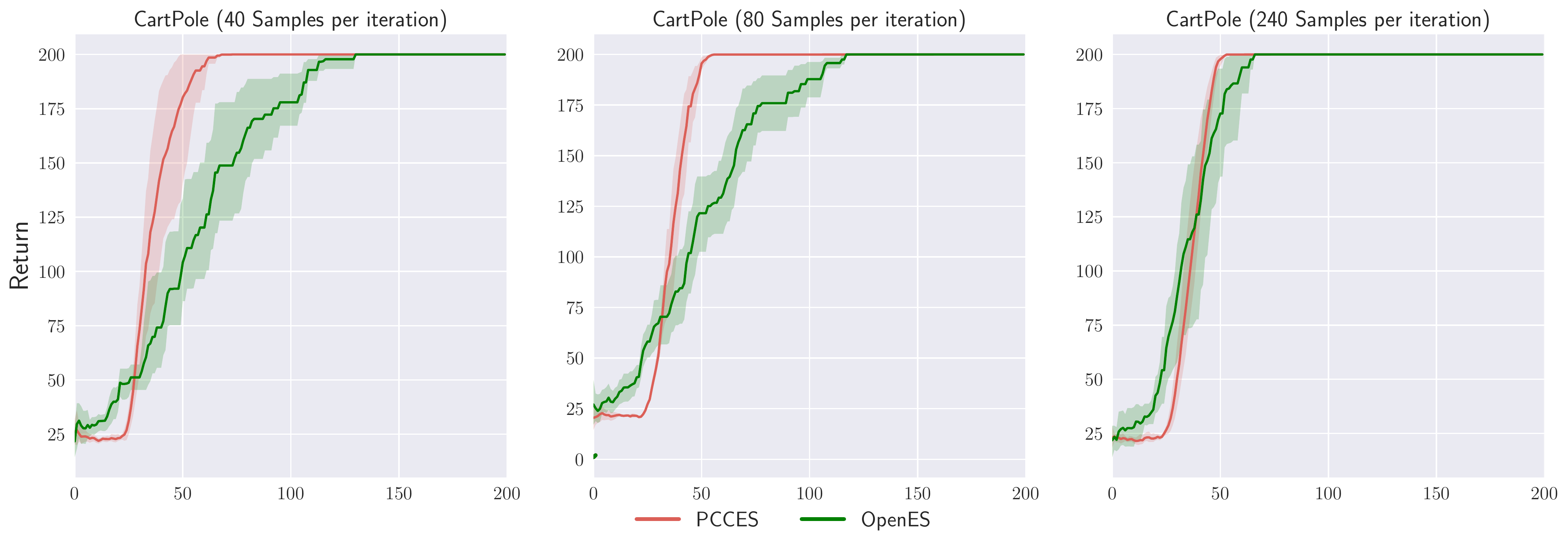}

\caption{CartPole-v0: PCCES and OpenES runs with 40, 80 and 240 samples per iteration}
\label{fig:cartpole_1}
\end{figure}

% Pendulum Plots
\begin{figure}[htp]

\centering
\includegraphics[width=\textwidth]{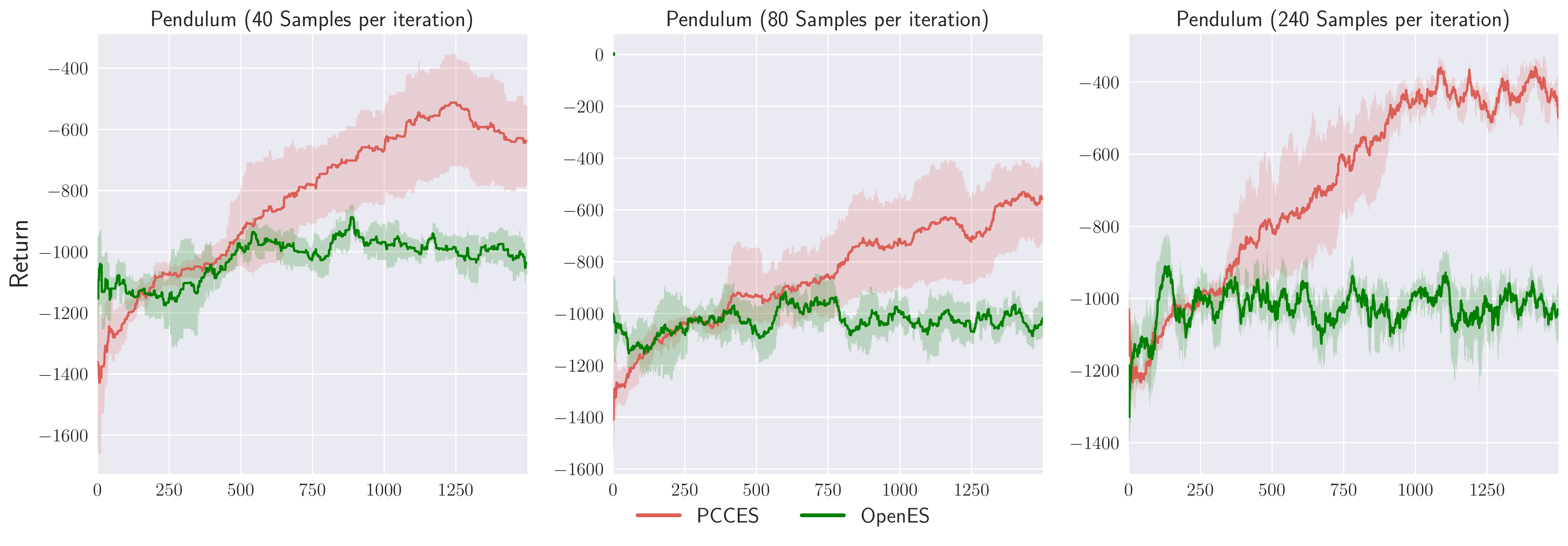}

\caption{Pendulum-v0: PCCES and OpenES runs with 40, 80 and 240 samples per iteration}
\label{fig:pendu_1}

\end{figure}

% MountainCarContinuous Plots
\begin{figure}[htp]

\centering
\includegraphics[width=\textwidth]{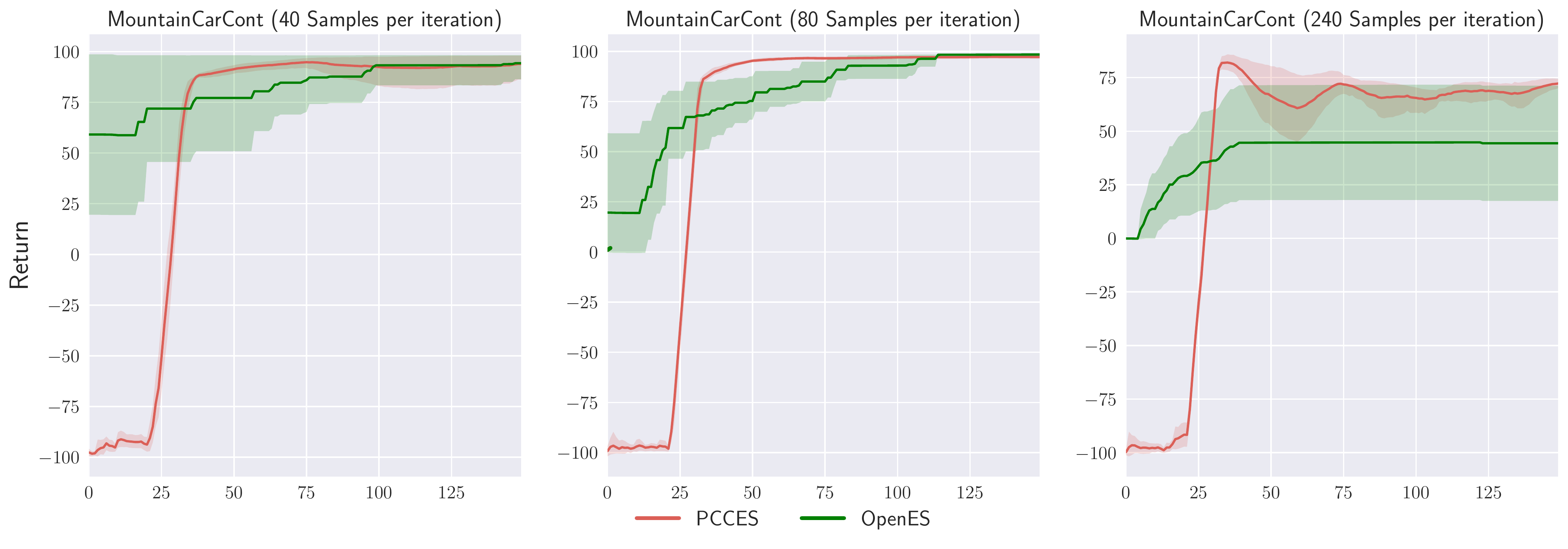}

\caption{MountainCarContinuous-v0: PCCES and OpenES runs with 40, 80 and 240 samples per iteration}
\label{fig:mcc_1}

\end{figure}

\begin{figure}[htp]
\centering
\includegraphics[width=\textwidth]{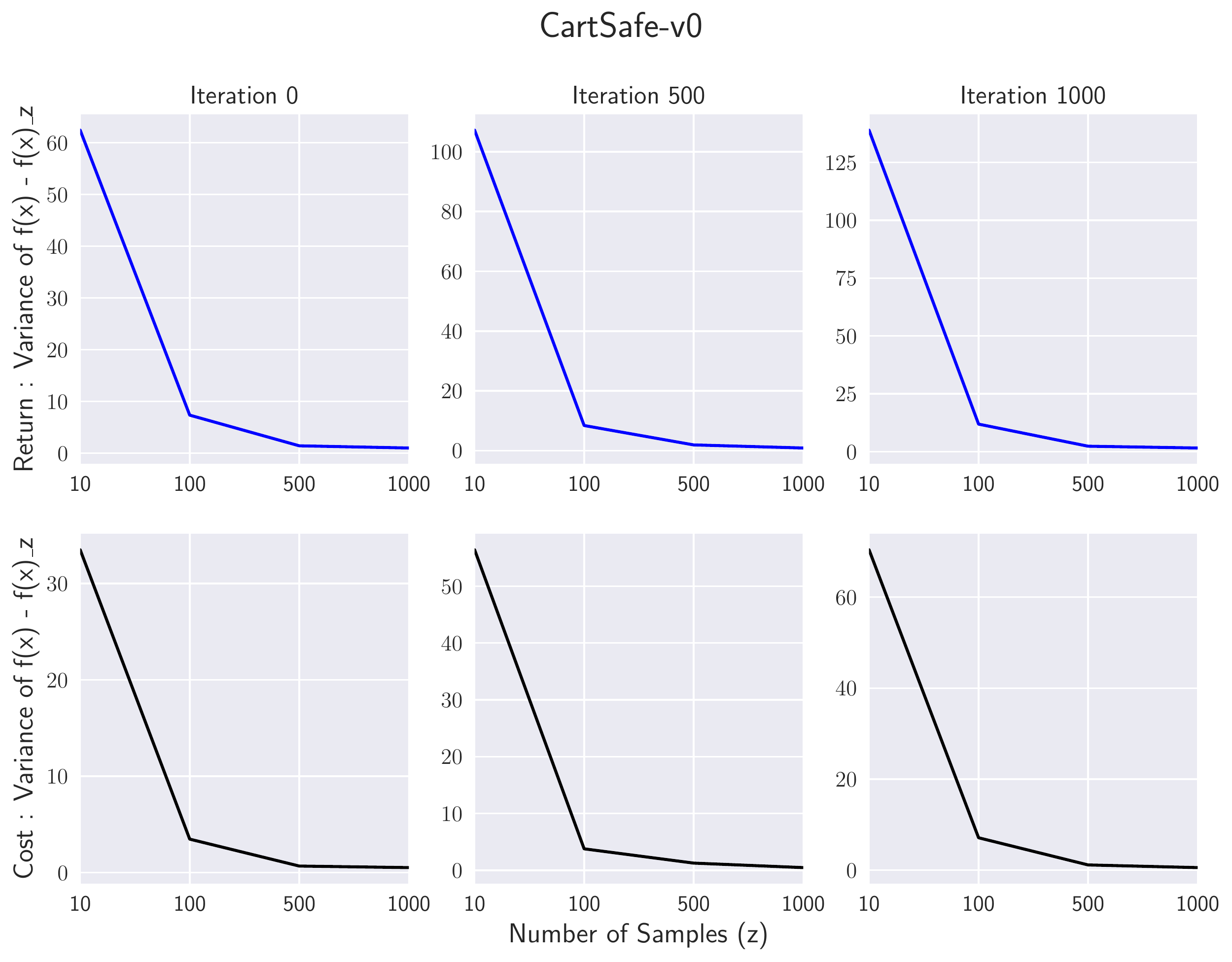}
\caption{Variance of function estimates vs number of samples for CartSafe-v0}
\label{fig:variance_1}

\end{figure}

\begin{figure}[htp]
\centering
\includegraphics[width=\textwidth]{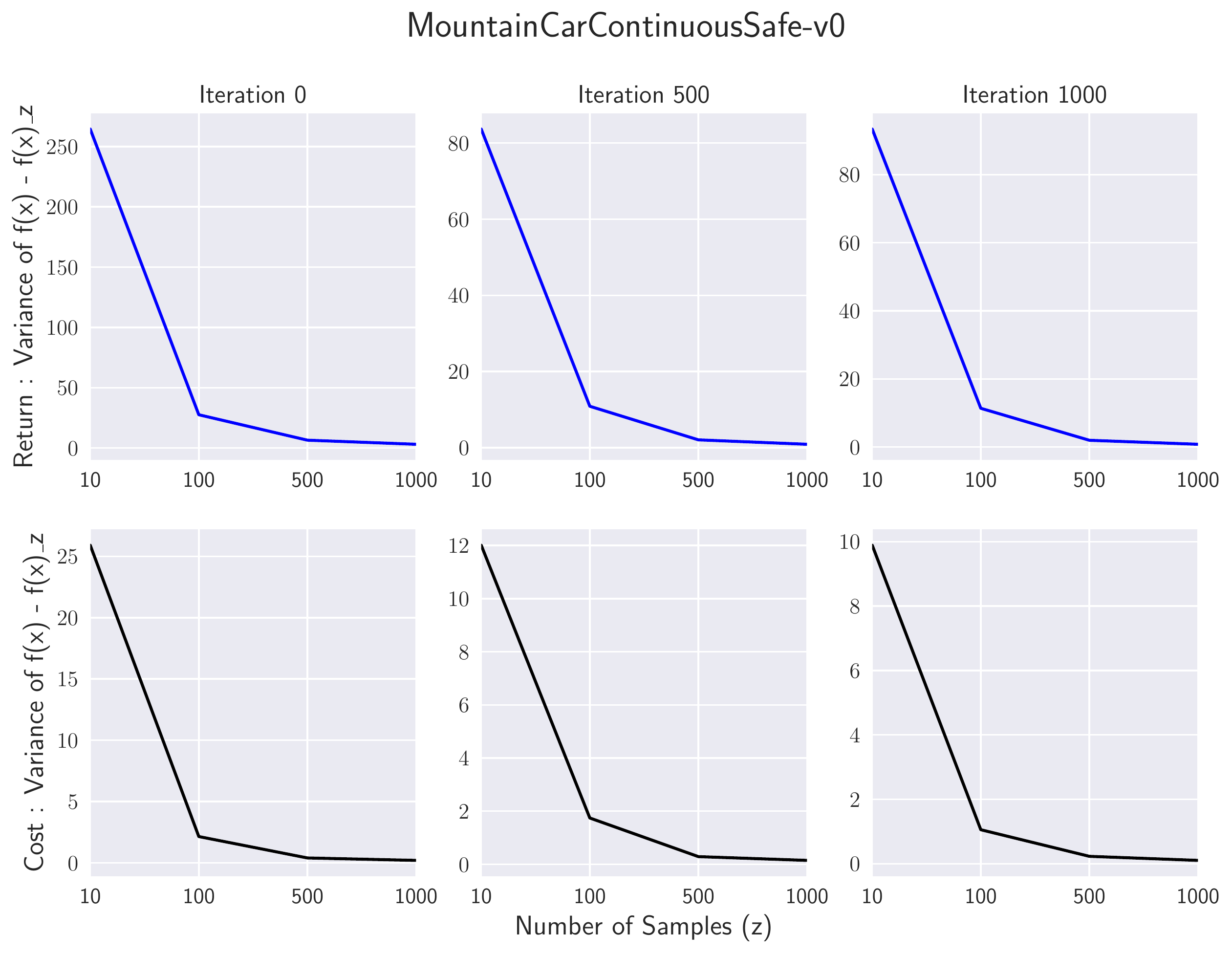}
\caption{Variance of function estimates vs number of samples for MountainCarContinuousSafe-v0}
\label{fig:variance_2}

\end{figure}

\begin{figure}[htp]
\centering
\includegraphics[width=\textwidth]{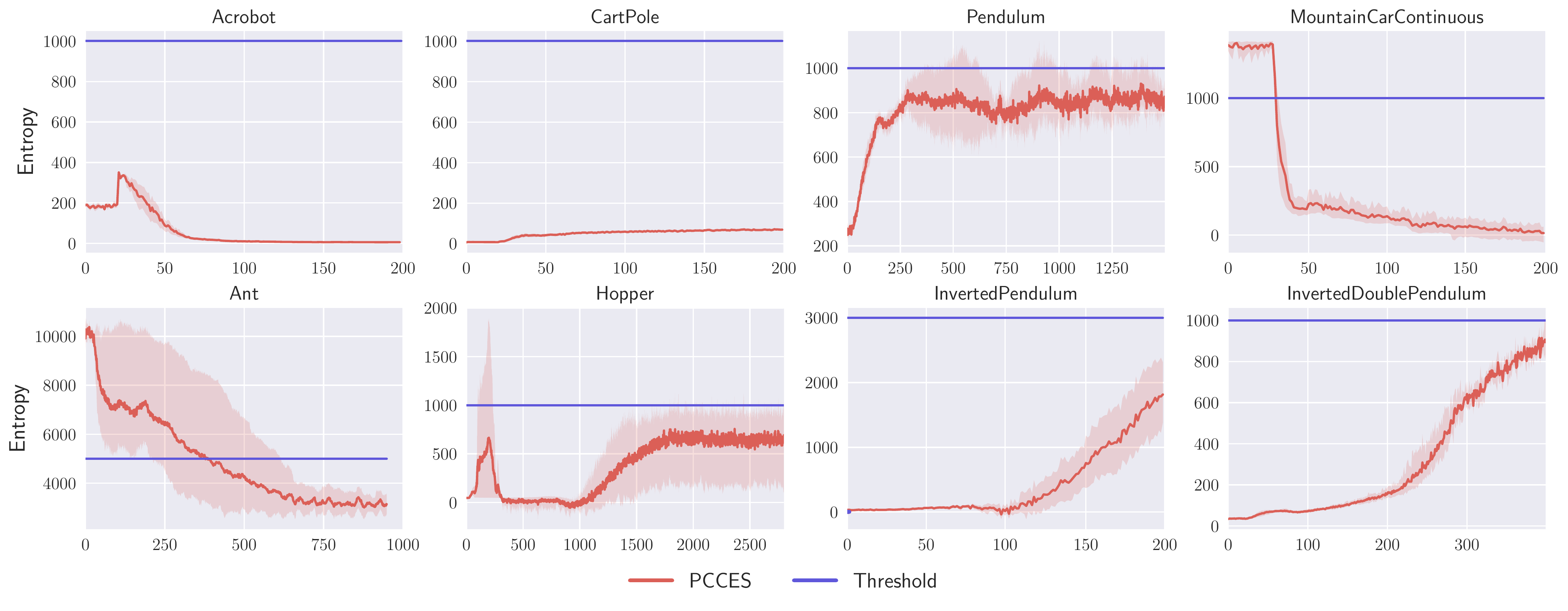}
\caption{Entropy vs upper threshold of the entropy constraint during training for constrained entropy maximization. PCCES consistently returns a solution which satisfies the constraints.}
\label{fig:entropy_const_exp}

\end{figure}

\newpage
 
\subsection{Sensitivity to Threshold}
We intend to include an ablation study for different thresholds. As a preliminary result, the following table reports the performance of PCCES on two environments with different threshold values after 300 updates and averaged over 5 seeds. This shows that the algorithm is not very sensitive of the choice of threshold.
\begin{table}[h]
\centering
\vspace{-3mm}
\begin{tabular}{|c|c|c|l|l|l|l|l|} 
\cline{1-4}\cline{6-8}
\multicolumn{4}{|c|}{MountainCarContinuousSafe-v0}            & \multicolumn{1}{c|}{} & \multicolumn{3}{c|}{CartSafe-v0}                                           \\ 
\cline{1-4}\cline{6-8}
Threshold & 5              & 15              & 20             &                       & 15               & 20                          & 35                \\ 
\cline{1-4}\cline{6-8}
Cost      & $2.9 \pm 1.1$  & $3.8 \pm 1.2$   & $3.1 \pm 2.09$ &                       & $7.8 \pm 6.43$   & $7.1 \pm 7.2$   & $14.1 \pm 13.8$   \\ 
\cline{1-4}\cline{6-8}
Return    & $94.9 \pm 0.4$ & $95.1 \pm 0.46$ & $94.6 \pm 0.4$ &                       & $227.1 \pm 17.9$ & $213.3 \pm 10.3$ & $229.7 \pm 18.5$  \\
\cline{1-4}\cline{6-8}
\end{tabular}
\label{table:threshold}
\end{table}

\subsection{Performance of CPO}
\begin{table}[h!]
\centering
\begin{tabular}{|c|c|c|c|l|l|l|l|l|} 
\cline{1-4}\cline{6-9}
\multicolumn{1}{|l|}{Clip Ratio} & \multicolumn{1}{l|}{0.05} & \multicolumn{1}{l|}{0.1} & \multicolumn{1}{l|}{0.2} &  & Step size & 1e-4           & 1e-5           & 3e-5            \\ 
\cline{1-4}\cline{6-9}
Cost                             & $70.4 \pm 4.2$            & $76.1 \pm 5.3$           & $61.3 \pm 13.1$          &  & Cost      & $75.2 \pm 7.1$ & $83.2 \pm 6.8$ & $68.6 \pm 4.1$  \\ 
\cline{1-4}\cline{6-9}
Return                           & $18.6 \pm 2.6$            & $19.7 \pm 0.8$           & $21.4 \pm 1.2$           &  & Return    & $15.36\pm 3.1$ & $17.2 \pm 0.1$ & $19.2 \pm 1.9$  \\
\cline{1-4}\cline{6-9}
\end{tabular}
\label{table:cpo_clip}
\end{table}

The above table reports the performance of CPO on Safexp-PointButton1-v0 after 300 epochs over 5 seeds for different clip ratios and step sizes. CPO fails to satisfy constraints, which is consistent with results in Safety Benchmarks \citep{Ray2019} (pg.18-19). 
 
 \subsection{Sensitivity to $\mu$}
  We observed changing $\mu$ by small amounts (0.0001 to 0.0002) had no significant effect on results. Though, large changes in $\mu$ (0.0001 to 0.1) could lead to a 
substantial difference in behavior. \\

\begin{figure}[h]
\centering
\includegraphics[width=0.8\textwidth]{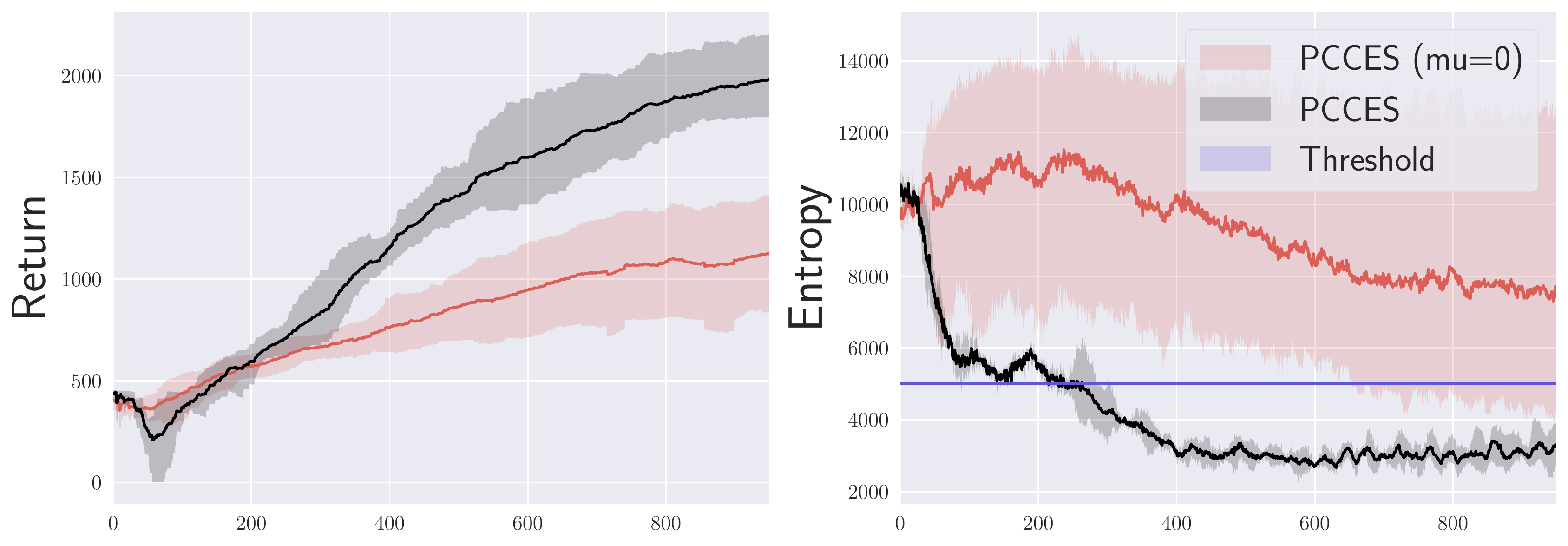}
\caption{\underline{Entropy $\mu=0$:} The above plot shows the performance of PCCES with $\mu=0$ (no constraint on entropy) vs PCCES with constraints on entropy for Ant-v0. These runs were conducted for 1000 updates and averaged over 5 seeds. This result clearly shows the benefit of $\mu > 0$.}
\end{figure}

%%%%%%%%%%%%%%%%%%%%%%%%%%%%%%%%%%%%%%%%%%%%%%%%%%%%%%%%%%%%%%%%%%%%%%%%%%%%%%%%
%%%%%%%%%%%%%%%%%%%%%%%%%%%%%%%%%%%%%%%%%%%%%%%%%%%%%%%%%%%%%%%%%%%%%%%%%%%%%%%%
\newpage
\section{Experiment Details}
\label{sec:exp_details}

\subsection{Hyperparameter Selection and Compute}
For the step size, we conducted a grid search over the following values [0.1, 0.01, 0.001, 0.05]. We conducted 3 runs for each environment with a different seed and selected the step size with the overall best performance. For $\alpha$ and $\beta$ we used the default values mentioned in \cite{maheswaranathan2018guided}. We fixed the initial standard deviation of sampling distribution to 1.0 for all experiments. The remaining hyperparameters are given in table \ref{tab:hyp_1}, \ref{tab:hyp_2} and \ref{tab:hyp_3}.

All experiments for PCCES, OpenES, ASEBO were conducted on servers with only CPU's, with number of CPU's per experiment varied from 5 - 30 depending on the availability of the compute. For CPO, RCPO and PPO, We distributed all runs across 4 CPUs per run and 1 GPU (various GPUs including GTX 1080 Ti, TITAN X, and TITAN V.) 

\begin{table}[h]
\centering
\begin{tabular}{|l|l|l|}
\hline
Hyperparameters          & PCCES       & OpenES      \\ \hline
L-2  coefficient                 & 0.0001      & 0.0001      \\ \hline
$\lambda$   & 40, 80, 240 & 40, 80, 240 \\ \hline
$\sigma_{k}^{\es}$ & 1.0         & 1.0         \\ \hline
$\sigma_{0}$                & 0.1         & 0.1         \\ \hline
decrease $\sigma$ rate          & 0.99        & -           \\ \hline
increase $\sigma$ rate      & 1.01        & -           \\ \hline
min $\sigma$            & 0.001       & -           \\ \hline
max $\sigma$            & 0.1         & -           \\ \hline
$\alpha$                    & 0.5         &     -        \\ \hline
$\beta$                     & 5.0         &     -       \\ \hline
discount factor                    & 0.99        &     -       \\ \hline
$\mu$            & 0.0001      &        -    \\ \hline
$m$               & 20          &       -      \\ \hline
$\kappa$             & 0.005         &       -        \\ \hline
\end{tabular}
\caption{Hyper parameters for PCCES and OpenES for control tasks}
\label{tab:hyp_1}
\end{table}

\begin{table}[h]
\centering
\begin{tabular}{|l|l|l|}
\hline
Environment               & Entropy Low & Entropy High \\ \hline
CartPole               & 0           & 1000         \\ \hline
Acrobot                & 0           & 1000         \\ \hline
MountainCarContinuous  & -1000       & 1000         \\ \hline
MountainCar             & 0           & 1000          \\ \hline
Pendulum               & 0           & 1000         \\ \hline
InvertedPendulum       & -1000       & 3000         \\ \hline
InvertedDoublePendulum & 0           & 1000          \\ \hline
Ant                    & 0           & 5000          \\ \hline
Hopper                  & 0           & 1000          \\ \hline
Walker                  & 0           & 2000          \\ \hline
CartPoleSafeDelayed-v0    & 0           & 2000         \\ \hline
MountainCarSafeDelayed-v0 & 0           & 2000         \\ \hline
Safexp-PointGoal1-v0      & -1000        & 5000            \\ \hline
Safexp-PointButton1-v0    & -1000        & 5000            \\ \hline
Safexp-CarGoal1-v0    & -1000        & 5000            \\ \hline
\end{tabular}
\caption{Entropy lower and upper bounds for all the envs}
\label{tab:hyp_2}
\end{table}

\begin{table}[h]
\centering
\begin{tabular}{|l|l|}
\hline
Environment               & Threshold \\ \hline
CartPoleSafeDelayed-v0    & 30        \\ \hline
MountainCarSafeDelayed-v0 & 10        \\ \hline
Safexp-PointGoal1-v0      & 25        \\ \hline
Safexp-PointButton1-v0    & 25        \\ \hline
Safexp-CarGoal1-v0    & 25        \\ \hline
\end{tabular}
\caption{Thresholds for the cost penalty in the constrained environments}
\label{tab:hyp_3}
\end{table}

\subsection{Software Libraries}
We thank the developers of Tensorflow \citep{tensorflow2015-whitepaper}, PyTorch \citep{pytorch2019}, OpenAI Gym \citep{brockman2016openai}, Numpy \citep{harris2020array}, RLLib \citep{liang2018rllib} and Matplotlib \citep{Hunter:2007}.

\end{document}